\documentclass{article} 
\usepackage{iclr2022_conference_mod,times}

\usepackage{hyperref}       
\usepackage{url}            
\usepackage{booktabs}       
\usepackage{nicefrac}       
\usepackage{microtype}      
\usepackage{subcaption}
\usepackage{graphicx}
\usepackage{amsmath,mathtools,amsfonts,amsthm} 
\usepackage{amssymb}
\usepackage{xcolor}
\usepackage{subcaption}
\usepackage{wrapfig}
\newtheorem{assumption}{Assumption}

\newtheorem{corollary}{Corollary}
\newtheorem{lemma}{Lemma}

\newtheorem{definition}{Definition}
\newtheorem{theorem}{Theorem}

\usepackage{tikz}
\usetikzlibrary{decorations.pathmorphing} 
\usetikzlibrary{matrix} 
\usetikzlibrary{arrows} 
\usetikzlibrary{calc} 

\tikzstyle{block} = [draw,rectangle,thick,minimum height=2em,minimum width=2em]
\tikzstyle{sum} = [draw,circle,inner sep=0mm,minimum size=2mm]
\tikzstyle{connector} = [->,thick]
\tikzstyle{line} = [thick]
\tikzstyle{branch} = [circle,inner sep=0pt,minimum size=1mm,fill=black,draw=black]
\tikzstyle{guide} = []

\title{Non-Parametric Neuro-Adaptive Control \\ Subject to Task Specifications}

\author{Christos K. Verginis   \\
Oden Institute for Computational Engineering and Sciences\\
University of Texas at Austin \\
Austin, TX78705, USA \\
\texttt{cverginis@utexas.edu} \\
\And
Zhe Xu \\
School for Engineering of Matter, \\Transport, and Energy  \\
Arizona State University \\
Tempe, AZ85281, USA \\
\texttt{xzhe1@asu.edu} \\
\AND
Ufuk Topcu \\
Oden Institute for Computational Engineering and Sciences\\
University of Texas at Austin \\
Austin, TX78705, USA \\
\texttt{utopcu@utexas.edu}
}

\iclrfinalcopy 
\begin{document}

\maketitle

\begin{abstract}
We develop a learning-based algorithm for the control
of autonomous systems governed by unknown, nonlinear dynamics to
satisfy user-specified spatio-temporal tasks expressed as signal temporal logic specifications. 
Most existing algorithms either assume certain parametric forms for the unknown dynamic terms or resort to unnecessarily large control inputs in order to provide theoretical guarantees. 
The proposed algorithm addresses these drawbacks by integrating neural-network-based learning with adaptive control. More specifically, the algorithm learns a controller, represented as a neural network, using training data that correspond to a collection of system parameters and tasks. These parameters and tasks are derived by  varying the nominal parameters and the spatio-temporal constraints of the user-specified task, respectively. 
 It then incorporates this neural network into an online closed-form adaptive control policy in such a way that the resulting behavior satisfies the user-defined task. 
The proposed algorithm does not use any a priori information on the unknown dynamic terms or any approximation schemes. We provide formal theoretical guarantees on the satisfaction of the task. Numerical experiments on a robotic manipulator and a unicycle robot demonstrate that the proposed algorithm guarantees the satisfaction of 50 user-defined tasks, and outperforms control policies that do not employ online adaptation or the neural-network controller. Finally, we show that the proposed algorithm achieves greater performance than standard reinforcement-learning algorithms in the pendulum benchmarking environment.
\end{abstract}

\section{Introduction} \label{sec:intro}

Learning and control of autonomous systems with uncertain dynamics is a critical and challenging topic that has been widely studied during the last decades. One can identify plenty of motivating reasons,	
ranging from uncertain geometrical or dynamical parameters and unknown exogenous disturbances to abrupt faults that significantly modify the dynamics. There has been, therefore, an increasing need for developing  control algorithms that do not rely on the underlying system dynamics. 
At the same time, such algorithms 
can be easily implemented on different, heterogeneous systems, since one does not need to be occupied with the tedious computation of the  dynamic terms.

There has been a large variety of works that tackle the problem of control of autonomous systems with uncertain dynamics, exhibiting, however, certain limitations. 
The existing algorithms are based on adaptive and learning-based approaches or the so-called funnel control (\cite{krstic1995nonlinear,vamvoudakis2010online,berger2018funnel,bechlioulis2014low,joshi2020asynchronous,capotondi2020online,bertsekas96,sutton2018reinforcement}). Nevertheless, adaptive control methodologies are restricted to system dynamics that can be linearly parameterized with respect to certain unknown parameters (e.g., masses, moments of inertia), assuming the system structure perfectly known; funnel controllers employ reciprocal terms that drive the control input to infinity when the tracking error approaches a pre-specified funnel function, 
creating thus unnecessarily large control inputs that might damage the system actuators. Data-based learning approaches either consider some system characteristic known (e.g., a nominal model, Lipschitz constants, or global bounds), or use neural networks to learn a tracking controller or the system dynamics; the correctness of such methods, however, relies on strong assumptions on the parametric approximation by the neural network and knowledge of the underlying radial basis functions. Finally, standard reinforcement-learning techniques (\cite{bertsekas96,sutton2018reinforcement}) usually assume certain state and/or time discretizations of the system and rely on exhaustive search of the state space, which might lead to undesirable transient properties (e.g., collision with obstacles while learning).

\subsection{Contributions and Significance}

This paper addresses the control of systems with continuous, \textit{unknown} nonlinear dynamics subject to task specifications expressed as signal interval temporal logic (SITL) constraints (\cite{lindemann2020efficient}).  Our main contribution lies in the  development of a  learning-based control 
algorithm that guarantees the accomplishment of a given task using only mild assumptions on the system dynamics. The algorithm draws a novel connection between adaptive control and learning with neural network representations, and consists of the following steps.
Firstly, it trains a neural network that aims to learn a controller that accomplishes a given task from data obtained off-line.  
Secondly, it calculates an open-loop trajectory that yields the execution of the task if followed by the system, while  neglecting the dynamics. Finally, we develop an online adaptive feedback control policy that uses the trained network to guarantee convergence to the open-loop trajectory and hence satisfaction of the task. Essentially, our approach builds on a combination of off-line trained controllers and on-line adaptations, which was recently shown to significantly enhance performance  with respect to single use of the off-line part (\cite{bertsekas2021lessons}). The proposed approach is particularly suitable for cases when engineering systems undergo purposeful modifications (e.g., the substitution of a motor/link in a robotic arm or exposure to new working environments) which might change their dynamics or operating conditions. In such cases, the goal is not to re-design new model-based controllers for the modified systems, but rather exploit the already designed ones and guarantee correctness via intelligent online adaptation policies. 


The major significance of our contribution is twofold. Firstly, we guarantee the theoretical correctness of the proposed algorithm by considering only mild conditions on the neural network, removing the long-standing assumptions on parametric approximations and boundedness of the estimation error. Secondly, we demonstrate via the experimental results the generality of the algorithm with respect to different tasks and system parameters. That is, 
the training data that we generate for the training of the neural network in the first step correspond
to tasks that are different, in terms of spatiotemporal specifications, from the given one to be executed\footnote{The task difference is illustrated in Section \ref{sec:exp res}.}. 
Additionally, we employ systems with different dynamic parameters to generate these data. We evaluate the proposed algorithm in numerous scenarios comprising different variations of the given task and different system dynamic parameters, which do not necessarily match the training data. We show that the algorithm, owing to its adaptation properties, is able to guarantee the satisfaction of the respective tasks in all the aforementioned scenarios by using the same neural network.

\subsection{Related work}

A large variety of previous works considers neuro-adaptive control with stability guarantees, focusing on the optimal control problem (\cite{vamvoudakis2010online,yang2020safe,cheng2007neural,fan2018robust,kiumarsi2017optimal,zhao2020finite,vrabie2009neural,sun2020continuous,kamalapurkar2015approximate,huang2018neuro,mo2019neuro,joshi2020asynchronous}). Nevertheless, the related works draw motivation from the neural network density property (see, e.g., (\cite{cybenko1989approximation}))\footnote{A sufficiently large neural network can approximate a continuous function arbitrarily well in a compact set.} and assume sufficiently small approximation errors and linear parameterizations of the unknown terms (dynamics, optimal controllers, or value functions), which is also the case with traditional adaptive control methodologies (\cite{krstic1995nonlinear,hong2009robust,chen2019nussbaum,huang2018adaptive}). 
This paper relaxes the aforementioned assumptions and proposes a \textit{non-parametric} neuro-adaptive controller, whose stability guarantees rely on a mild boundedness condition of the closed-loop system state that is driven by the learned controller. The proposed approach exhibits similarities with (\cite{liu1994integrated}), which employs off-line-trained neural networks with online feedback control, but fails to provide convergence guarantees.

Other learning-based related works include modeling with Gaussian processes~(\cite{capotondi2020online,leahy2019control,jain2018learning,berkenkamp2015safe}), or use neural networks~(\cite{ma2020stlnet,shah2018bayesian,yan2021neural,liu2021recurrent,hahn2020teaching,cai2021modular,wang2020continuous,camacho2019towards,hahn2020teaching,riegel2020logical,hu2020reach,ivanov2019verisig}) to accomplish reachability, verification or temporal logic specifications. Nevertheless, the aforementioned works either use partial information on the underlying system dynamics, or do not consider them at all. In addition, works based on Gaussian processes usually propagate the dynamic uncertainties, possibly leading to conservative results. Similarly, data-driven model-predictive control techniques (\cite{nubert2020safe,maddalena2020neural}) use data to over-approximate  additive disturbances or are restricted to linear systems.

Control of unknown nonlinear continuous-time systems has been also tackled in the literature by using funnel control, without necessarily using off-line data or dynamic approximations (\cite{berger2018funnel,bechlioulis2014low,lindemann2017prescribed,verginis2018timed,verginis2021kdf}). Nevertheless, funnel controllers usually depend on reciprocal time-varying barrier functions that drive the control input to infinity
when the error approaches a pre-specified funnel, creating thus unnecessarily large control inputs that might damage the system actuators.

\section{Preliminaries and Problem Formulation}


\subsection{Signal interval temporal logics}

Let  $y:\mathbb{R}_{\geq 0}\to\mathbb{R}^n$ be a continuous-time signal. Signal interval temporal logic (SITL)   consists  of  predicates $\mu$ that  are  obtained  after  evaluation  of  a continuously  differentiable  predicate  function $h:\mathbb{R}^n\to\mathbb{R}$ (\cite{lindemann2020efficient}).   For $\zeta\in\mathbb{R}^n$,  let $\mu \coloneqq \top$ if $h(\zeta)\geq0$ and $\mu\coloneqq \bot$ if $h(\zeta)<0$.  
The SITL syntax is given by 
\begin{align*}
	\varphi \coloneqq \top \ | \ \mu \ | \ \neg \varphi \ | \ \varphi_1 \land \varphi_2 \ | \varphi_1 U_{[a,b]} \varphi_2,
\end{align*}
where $\varphi_1$ and $\varphi_2$ are
SITL formulas, $[a,b] \subset \mathbb{R}_{\geq 0}$, with $b > a$ is a time interval, and $U_{[a,b]}$ encodes  the  until  operator.  Define the eventually and always operators as $F_{[a,b]}\varphi \coloneqq \top U_{[a,b]} \varphi$ and $G_{[a,b]}\varphi \coloneqq \neg ¬F_{[a,b]}\neg \varphi$.  The satisfaction relation $(y,t)\models \varphi$ denotes that the signal $y:\mathbb{R}_{\geq 0}\to\mathbb{R}^n$ satisfies $\varphi$ at time $t$.  The SITL semantics are recursively given by $(y,t)\models \mu$ if and only if $h(y(t))\geq 0$, $(y,t)\models \neg \varphi$ if and only if $\neg((y,t)\models \varphi)$, $(y,t)\models \varphi_2 \land \varphi_2$ if and only if $(x,t) \models \varphi_1$ and $(y,t) \models \varphi_2$, and $(y,t) \models \varphi_1 U_{[a,b]} \varphi_2$ if and only if there exists a $t_1\in[t+a,t+b]$ such that $(y,t_1)\models \varphi_2$ and $(y,t_2)\models \varphi_1$, for all $t_2 \in [t,t_1]$.   A  formula $\varphi$ is  satisfiable  if there exists a $y:\mathbb{R}_{\geq 0}\to \mathbb{R}^n$ such that $(y,0)\models\varphi$. 

\subsection{Problem Statement} \label{sec:problem statement}

Consider a continuous-time dynamical system governed by the $2$nd-order 
continuous-time  dynamics  
	\begin{align}  \label{eq:dynamics}
		\ddot{x} &= f(\bar{x},t) + g(\bar{x},t)u(\bar{x},t), 
	\end{align}
where $\bar{x}\coloneqq [x^\top,\dot{x}^\top]^\top \in \mathbb{R}^{2n}$, $n \in \mathbb{N}$, is the  system state, assumed available for measurement,
and $u:\mathbb{R}^{2n}\times[0,\infty)\to\mathbb{R}^n$ is the time-varying feedback-control input.
The terms $f(\cdot)$ and $g(\cdot)$ are nonlinear vector fields that are locally Lipschitz in $\bar{x}$ over $\mathbb{R}^{2n}$ for each fixed $t\geq 0$, and uniformly bounded in $t$ over $[0,\infty)$ for each fixed $\bar{x}\in\mathbb{R}^{2n}$. The dynamics (\ref{eq:dynamics}) comprise a large class of nonlinear dynamical systems (\cite{zhong2020unsupervised,yu1996neural,doya1997efficient}) that capture contemporary engineering problems in mechanical, electromechanical and power electronics applications, such as rigid/flexible robots, induction motors and DC-to-DC converters, to name a few. The continuity in time and state provides a direct link to the actual underlying system, and we further do not require any time or state discretizations. 

We consider that $f(\cdot)$ and $g(\cdot)$ are completely unknown; we do not assume any knowledge of the structure, Lipschitz constants, or bounds, and we do not use any scheme to approximate them.
Nevertheless, we do assume that $g(\bar{x},t)$ is positive definite:
\begin{assumption} \label{ass:g pd}
	The matrix $g(\bar{x},t)$ is positive definite, for all $(\bar{x},t)\in \mathbb{R}^{2n} \times [0,\infty)$. 
\end{assumption}
Such assumption is a sufficiently controllability condition for (\ref{eq:dynamics}); intuitively, it states that the multiplier of $u$ (the input matrix) does not change the direction imposed to the system by the underlying control algorithm.
Systems not covered by (\ref{eq:dynamics}) or Assumption \ref{ass:g pd} consist of underactuated or non-holonomic systems, such as unicycle robots or underactuated aerial vehicles. Nevertheless, we extend our results for a non-holonomic unicycle vehicle in Section \ref{sec:control design}. Moreover, the $2$nd-order model (\ref{eq:dynamics}) can be easily extended to account for higher-order integrator systems (\cite{slotine1991applied}).

Consider now a time-constrained task expressed as an SITL formula $\varphi$ over $x$. The objective of this paper is to construct a time-varying feedback-control algorithm $u(\bar{x},t)$ such that the output of the closed-loop system (\ref{eq:dynamics}) satisfies $\varphi$, i.e., $(x(t),t) \models \varphi$.

\section{Main Results}

This section describes the proposed algorithm, which consists of three steps. Firstly, it learns a controller, represented as a neural network, using training data that correspond to a collection of different tasks and system parameters.
Secondly, it uses formal methods tools to compute an \textit{open-loop} trajectory that satisfies the given task. Finally, we design an adaptive, time-varying feedback controller that uses the neural-network approximation and guarantees tracking of the open-loop trajectory, consequently achieving satisfaction of the task.

\subsection{Neural-network learning} \label{subsec:NN}

We assume the existence of offline data from a finite set of $T$ system trajectories that satisfy a collection of SITL tasks, including the one modeled by $\varphi$, and possibly produced by systems with different dynamic parameters. The data from each trajectory $i\in\{1,\dots,T\}$ comprise a finite set of triplets $\{\bar{x}_s(t),t,u_s(t)\}_{t\in \mathcal{T}_i}$, where $\mathcal{T}_i$ is a finite set of time instants, $\bar{x}_s(t)\in\mathbb{R}^{2n}$ are system states, and $u_s(t) \in \mathbb{R}^n$ are the respective control inputs, compliant with the dynamics (\ref{eq:dynamics}). 
We use the data to train a neural network in order to approximate the respective controller $u(\bar{x},t)$. More specifically, we use the pairs $(\bar{x}_s(t),t)_{t\in\mathcal{T}_i}$ as input to a neural network, and $u_s(t)_{t\in\mathcal{T}_i}$ as the respective output targets, for all trajectories $i\in\{1,\dots,T\}$. For given $\bar{x} \in \mathbb{R}^{2n},t \in \mathbb{R}_{\geq 0}$, we denote by $u_\textup{nn}(\bar{x},t)$ the output of the neural network. Note that the controller $u(\bar{x},t)$, which the neural network aims to approximate, is not associated to the specific task modeled by $\varphi$ and mentioned in Section \ref{sec:problem statement}, but a collection of SITL tasks. Therefore, we do not expect the neural network to learn how to accomplish this specific task $\varphi$, but rather to be able to adapt to the entire collection of tasks. 
This is an important attribute of the proposed scheme, since it can generalize over the SITL tasks; that is, the specific task $\varphi$ to be accomplished can be any task of the aforementioned collection. The proposed algorithm, consisting of the trained neural network and the online feedback-control policy - illustrated in the next sections - is still able to guarantee its satisfaction. 

\subsection{Open-loop trajectory} \label{sec:ol traj}

The next step is the computation of an open-loop trajectory $p_\textup{d}:\mathbb{R}_{\geq 0} \to \mathbb{R}^n$ that satisfies $\varphi$, i.e., $(p_\textup{d}(t),t)  \models \varphi$. For this computation, we use the recent work~(\cite{lindemann2020efficient}), which proposes an efficient planning algorithm using automata to construct a set of bounded trajectories that satisfy an SITL formula. In order to accommodate the temporal aspect of an SITL formula, such a trajectory can be represented as a finite prefix followed by the infinite repetition of a finite suffix, i.e., 
	$p_\textup{d}(t) = p_\textup{d}(0:t_{f_1})\Big|[p_\textup{d}(t_{f_1}:t_{f_2})]^\omega$.
Here, $p_\textup{d}(0:t_{f_1})$ and $p_\textup{d}(t_{f_1}:t_{f_2})$ denote the trajectories $p_\textup{d}(t)$ from $0$ to $t_{f_1}$ and from $t_{f_1}$ to $t_{f_2}$ respectively, for some time instants $t_{f_1} > 0$, $t_{f_2} > t_{f_1}$. The operator $\cdot|\cdot$ denotes trajectory concatenation and the superscript $\omega$ denotes infinite repetition. 
Note that the training data of Section \ref{subsec:NN} are assumed to follow this prefix-suffix form; each trajectory is assumed to consist of a finite prefix followed by some repetitions of a finite suffix, i.e.,  $\max\{ \mathcal{T}_i \} > \kappa t_{f_2}$ for some integer $\kappa > 2$ for all trajectories $i\in\{1,\dots,T\}$. 
The procedure of computing  $p_\textup{d}(t)$ does \textit{not} take into account the dynamics (\ref{eq:dynamics}); unlike the training data used in Section \ref{subsec:NN}, $p_\textup{d}(t)$ is a geometric trajectory in $\mathbb{R}^n$ that satisfies the given task. More details regarding the procedure are out of scope of this paper and can be found in (\cite{lindemann2020efficient}).

\subsection{Feedback control design} \label{sec:control design}

As mentioned in Section \ref{subsec:NN}, we do not expect the neural-network controller to accomplish the given task, since the system (1) is trained on potentially different tasks and different system parameters, and (2) the neural network provides only an \textit{approximation} of a stabilizing controller; potential deviations in certain regions of the state space might lead to instability. Moreover, the neural-network controller has no error feedback with respect to the open-loop trajectory $p_{\textup{d}}$; such feedback is substantial in the stability of control systems with dynamic uncertainties. 
Therefore, this section is devoted to the design of a feedback-control policy to track the trajectory $p_\textup{d}(t)$ by using the output of the trained neural network (see Fig. \ref{fig:blck d + unicycle}a). The goal is to drive the error $e \coloneqq x - p_\textup{d}$ to zero.
We first impose an assumption on the closed-loop system trajectory that is driven by the neural network's output. 

\begin{assumption} \label{ass:bound}
	The output $u_\textup{nn}(\bar{x},t)$ of the trained neural network satisfies 
	\begin{align} \label{eq:assumption bound}
		\| f(\bar{x},t) + g(\bar{x},t)u_\textup{nn}(\bar{x},t) \| \leq d \|\bar{x}\| + B 
	\end{align}
	for positive constants $d$, $B$, for all $\bar{x} \in \mathbb{R}^{2n}$, $t\geq 0$. 
\end{assumption}

Intuitively, Assumption \ref{ass:bound} states that the neural-network controller $u_\textup{nn}(\bar{x},t)$ is able to maintain the \textit{boundedness} of the system state by the  constants $d$, $B$, which are considered to be \textit{unknown}. The assumption is motivated by the property of neural networks to approximate a continuous function arbitrarily well in a compact domain for a large enough number of neurons and layers~(\cite{cybenko1989approximation})\footnote{For simplicity, we consider that
(\ref{eq:assumption bound}) holds globally, but it can be extended to hold in a compact set.}.
Loosely speaking, since the collection of SITL tasks, which the neural network is trained with, correspond to bounded trajectories, the system states are expected to remain bounded. Since $f(\bar{x},t)$, and $g(\bar{x},t)$ are continuous in $\bar{x}$ and bounded in $t$, they are also expected to be bounded as per  (\ref{ass:bound}).
Contrary to the related works  (e.g.,~(\cite{vamvoudakis2010online,yang2020safe,cheng2007neural,fan2018robust})), however, we do not adopt approximation schemes for the system dynamics and we do not impose restrictions on the size of $d$, $B$. Moreover, Assumption \ref{ass:bound} does not imply that the neural network controller $u_\textup{nn}(\bar{x},t)$ guarantees tracking of the open-loop trajectory $p_\textup{d}$. Instead, it is merely a growth condition.
Additionally, note that inequality \ref{eq:assumption bound} does not depend specifically on any of the SITL tasks that the neural network is trained with. We exploit this property in the control design and achieve task generalization; that is, the open-loop trajectory $p_\textup{d}$ to be tracked (corresponding to the task $\varphi$) can be any of the tasks that the neural network is trained with.
 

We now define the feedback-control policy. 
Consider the adaptation variables $\hat{\ell}_1$, $\hat{\ell}_2$,  corresponding to upper bounds of $d$, $B$ in (\ref{eq:assumption bound}), with $\hat{\ell}_1(0) > 0$, $\hat{\ell}_2(0) > 0$. 
We design first a reference signal for $\dot{x}$ as 
\begin{align} \label{eq:v d}
	v_\textup{d} \coloneqq \dot{p}_\textup{d} - k_1 e,
\end{align}
that would stabilize the subsystem $\|e\|^2$, where $k_1$ is a positive control gain constant. Following the back-stepping methodology (\cite{krstic1995nonlinear}),  we define next the respective error $e_v \coloneqq \dot{x} - v_\textup{d}$ and 
design the neural-network-based adaptive control law as
\begin{subequations} \label{eq:control as}
	\begin{align}
		&u(\bar{x},\hat{\ell}_1,\hat{\ell}_2,t) = u_\textup{nn}(\bar{x},t) - k_2 e_v - \hat{\ell}_1 e_v - \hat{\ell}_2\hat{e}_v \\
		&\dot{\hat{\ell}}_1 = k_{\ell_1} \|e_v\|^2, \ \ \dot{\hat{\ell}}_2 = k_{\ell_2} \|e_v\|		
	\end{align} 
	where $k_2,  k_{\ell_1},  k_{\ell_2}$ are positive constants, and $\hat{e}_v = \frac{e_v}{\|e_v\|}$ if $e_v\neq0$, and $\hat{e}_v = 0$ if $e_v = 0$.
\end{subequations}
%
The control design is inspired by adaptive control methodologies~(\cite{krstic1995nonlinear}), where the time-varying gains $\hat{\ell}_1(t)$, $\hat{\ell}_2(t)$, adapt to the unknown dynamics and counteract the effect of $d$ and $B$ in (\ref{eq:assumption bound}) in order to ensure closed-loop stability. 
Note that the policy (\ref{eq:v d}), (\ref{eq:control as}) does not use any information on the system dynamics $f(\cdot)$, $g(\cdot)$ or the constants $B$, $d$.
The tracking of $p_\textup{d}$ is guaranteed by the following theorem, whose proof is given in Appendix \ref{app:A}.

\begin{figure}
	\begin{subfigure}[b]{0.5\textwidth}
		\centering
		\includegraphics[width=.6\textwidth]{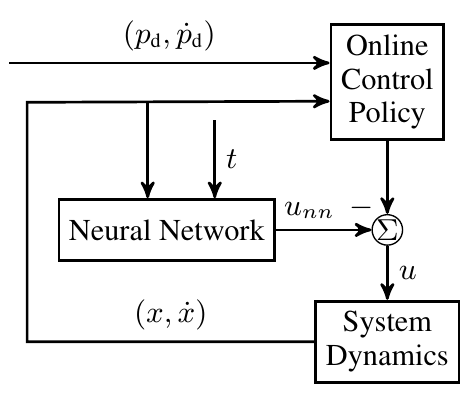}
		\caption{}
		\label{fig:block diagram}
	\end{subfigure}
	~
	\begin{subfigure}[b]{0.5\textwidth}
		\centering
		\includegraphics[width=.45\textwidth]{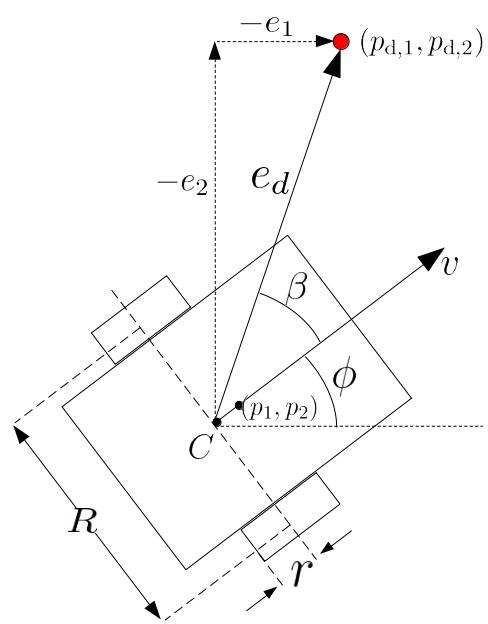}
		\caption{}
		\label{fig:unicycle}
	\end{subfigure}
	\caption{(a): Block diagram of the proposed algorithm. (b): A unicycle vehicle.}
	\label{fig:blck d + unicycle}
\end{figure}

\begin{theorem} \label{th:as}
	Let a system evolve according to (\ref{eq:dynamics}) and let an open-loop trajectory $p_\textup{d}(t)$ that satisfies a given SITL task modeled by $\varphi$. Under Assumption \ref{ass:bound}, the control algorithm (\ref{eq:control as}) guarantees $\lim_{t\to\infty}(e(t),e_v(t)) = 0$, 
	as well as the boundedness of all closed-loop signals.
\end{theorem}

Note that, contrary to works in the related literature (e.g.,~(\cite{verginis2020asymptotic,bechlioulis2014low})), we do not impose reciprocal terms in the control input that grow unbounded in order to guarantee closed-loop stability. The resulting controller is essentially a simple linear feedback on $(e(t),e_v(t))$ with time-varying adaptive control gains, accompanied by the neural network output that ensures the growth condition (\ref{eq:assumption bound}). The positive gains $k_1$, $k_2$, $k_{\ell_1}$, $k_{\ell_2}$ do not affect the   stability results of Theorem \ref{th:as}, but might affect the evolution of the closed-loop system; e.g., larger gains lead to faster convergence but possibly larger control inputs.

\textbf{Extension to non-holonomic unicycle dynamics}

As mentioned in Section \ref{sec:intro}, the dynamics \ref{eq:dynamics} do not represent all kinds of systems, with one particular example being when non-holonomic constraints are present. In such cases, the control law design (\ref{eq:control as}) and Theorem \ref{th:as} no longer hold. In this section, we extend the control policy to account for unicycle vehicles subject to first-order non-holonomic constraints. 
More specifically, we consider the dynamics
\begin{subequations} \label{eq:dynamics unicycle}
	\begin{align} 
		&\dot{p}_1 = v \cos \phi, \hspace{2mm} 
		\dot{p}_2 = v \sin \phi, \hspace{2mm} 
		\dot{\phi} = \omega \\
		&\hspace{12mm} M \ddot{\theta} = u + f_\theta(\bar{x},t) 
	\end{align}
\end{subequations}
where $x = [p_1,p_2,\phi]^\top\in \mathbb{R}^3$ are the unicycle's position and orientation, $(v,\omega)$ are its linear and angular velocity (see Fig. \ref{fig:blck d + unicycle}b), $\theta\coloneqq[\theta_R,\theta_L]^\top \in \mathbb{R}^2$ are its wheel's angular positions, and $u = [u_R,u_L]^\top \in \mathbb{R}^2$ are the wheel's torques, representing the control input. 
The unicycle vehicle is subject to the non-holonomic constraint $\dot{p}_1\sin\phi - \dot{p}_2\cos\phi = 0$,
which implies that the vehicle cannot move laterally. 
Additionally, $M\in\mathbb{R}^{2\times2}$ is the vehicle's inertia matrix, which is symmetric and positive definite, and $f_\theta(\cdot)$ is a function representing friction and external disturbances.
 The velocities satisfy the relations $v = \frac{r}{2}(\dot{\theta}_R + \dot{\theta}_L)$, $\omega = \frac{r}{2R}(\dot{\theta}_R - \dot{\theta}_L))$, where $r$ and $R$ are the wheels' radius and axle length, respectively. The terms $r$, $R$, $M$, and $f_\theta(\cdot)$ are considered to be \textit{completely unknown}. As before, the goal is for the vehicle's position $p \coloneqq [p_1,p_2]^\top$ to track the desired trajectory $p_\textup{d} = [p_{\textup{d},1},p_{\textup{d},2}]^\top \in \mathbb{R}^2$, which is output from the procedure described in Sec. \ref{sec:ol traj}. Towards that end, we define the error variables $e_1  \coloneqq p_1 - p_{\textup{d},1}$, $e_2 \coloneqq p_2 - p_{\textup{d},2}$, $e_d \coloneqq \|p-p_\textup{d}\|$, as well as the angle $\beta$ measured from the the longitudinal axis of the vehicle, i.e., the unicycle's direction vector $[\cos\phi,\sin\phi]$, to the error vector $-[e_1,e_2]$ (see Fig. \ref{fig:unicycle}). The angle $\beta$ can be derived by using the cross product between the aforementioned vectors, i.e., $	e_d\sin(\beta) = [\cos\phi,\sin\phi]\times[-e_1,-e_2] = {e_1}\sin\phi - {e_2}\cos\phi$.
The purpose of the control design, illustrated next, is to drive $e_d$ and $\beta$ to zero. By differentiating the latter and using (\ref{eq:dynamics unicycle}) as well as the relations $e_1 = -e_d \cos(\phi+\beta)$, $e_2 = -e_d \sin(\phi + \beta)$ (see Fig. \ref{fig:unicycle}), we derive 
\begin{subequations} \label{eq:ed beta dot}	
	\begin{align}
		\dot{e}_d &= -v\cos\beta + \dot{p}_{\textup{d},1}\cos(\phi+\beta) + \dot{p}_{\textup{d},2}\sin(\phi+\beta)\\
		\dot{\beta} &= -\omega + \frac{\sin\beta}{e_d}v - \frac{\dot{p}_{\textup{d},1}}{e_d}\sin(\phi+\beta) + \frac{\dot{p}_{\textup{d},2}}{e_d}\cos(\phi+\beta)
	\end{align}
\end{subequations}
In view of (\ref{eq:ed beta dot}),
we set reference signals for the vehicle's velocity as 
\begin{subequations} \label{eq:v_des unicycle}
\begin{align}
	v_\textup{d} \coloneqq& \frac{1}{\cos(\beta)}(\dot{p}_{\textup{d},1}\cos(\beta+\phi) + \dot{p}_{\textup{d},2}\sin(\beta+\phi) + k_d e_d) \\
	\omega_\textup{d} \coloneqq& -\frac{\sin(\phi) \dot{p}_{\textup{d},1} }{\cos(\beta)e_d} + \frac{\cos(\phi)\dot{p}_{\textup{d},2}}{\cos(\beta)e_d} + k_d\tan\beta  + k_\beta \beta
\end{align} 
\end{subequations}
where $k_d$, $k_\beta$ are positive gains, aiming to create exponentially stable subsystems via the terms $k_de_d$ and $k_\beta \beta$. We define next the respective velocity errors $e_v \coloneqq v - v_\textup{d}$, $e_\omega \coloneqq \omega - \omega_\textup{d}$
and design the adaptive and neural-network-based control input as $u(\bar{x},\hat{d},t) \coloneqq [\frac{u_S+u_D}{2},\frac{u_S-u_D}{2}]^\top + u_{\textup{nn}}(\bar{x},t)$, with 
\begin{subequations} \label{eq:control as unicycle}
\begin{align}
	u_S &\coloneqq \hat{\ell}_{v}\dot{v}_{\textup{d}} - (k_v + \hat{\ell}_1)e_v - \hat{\ell}_2\hat{e}_v + e_d\cos\beta - \beta\frac{\sin\beta}{e_d} \label{eq:control as unicycle u_s}\\
	u_D &\coloneqq \hat{\ell}_{\omega}\dot{\omega}_{\textup{d}} - (k_\omega + \hat{\ell}_1)e_\omega- \hat{\ell}_2\hat{e}_\omega  + \beta \label{eq:control as unicycle u_D}
\end{align}	
\begin{align}
	\dot{\hat{\ell}}_{v} &\coloneqq -k_{v}e_v\dot{v}_\textup{d}, \hspace{10mm}
	\dot{\hat{\ell}}_{\omega} \coloneqq -k_{\omega}e_{\omega}\dot{\omega}_\textup{d} \\	
	\dot{\hat{\ell}}_1 &\coloneqq k_1(e_v^2 + e_\omega^2) \hspace{6mm} \dot{\hat{\ell}}_2 \coloneqq k_2(|e_v|+|e_\omega|)
\end{align}
\end{subequations}
where $\hat{\ell}_{v}$, $\hat{\ell}_{\omega}$, $\hat{\ell}_i$ are adaptation variables (similar to (\ref{eq:control as})), with $\hat{\ell}_v(0)>0$, $\hat{\ell}_\omega(0)>0$ and $k_{v}$, $k_{\omega}$, $k_i$, are positive gains, $i\in\{1,2\}$;  $\hat{e}_a$, with $a\in\{v,\omega\}$, is defined as $\hat{e}_a = \frac{e_a}{|e_a|}$ if $e_a\neq 0$ and $\hat{e}_a = 0$ otherwise. 
We now re-state Assumption \ref{eq:assumption bound} to apply for the unicycle analysis as follows. 
\begin{assumption} \label{ass:bound unicycle} 
	The output $u_\textup{nn}(\bar{x},t)$ of the trained neural network satisfies $		\|u_\textup{nn}(\bar{x},t) + f_\theta(\bar{x},t) \|  \leq  d\|\bar{x}\| + B$,
for positive, unknown constants $d$, $B$.
\end{assumption}
Similar to assumption \ref{ass:bound}, assumption \ref{ass:bound unicycle} is merely a growth-boundedness condition by the unknown constants $d$ and $B$. The stability of the proposed scheme is provided in the next corollary, whose proof is found in Appendix \ref{app:A}. 
\begin{corollary} \label{th:as unicycle}
	Let the unicycle system (\ref{eq:dynamics unicycle}) and let an open-loop trajectory $p_\textup{d}(t)$ that satisfies a given SITL task modeled by $\varphi$. Assume that $\beta(t)\in(-\bar{\beta},\bar{\beta})$, ${|\dot{p}_{\textup{d},1} \sin\phi - \dot{p}_{\textup{d,2}}\cos\phi |} < e_d\alpha_1$, ${\sin\beta} < e_d\alpha_2$  for positive constants $\bar{\beta} < \frac{\pi}{2}$, $\alpha_1$, $\alpha_2$ and all $t\geq 0$. Under Assumption \ref{ass:bound unicycle}, the control policy (\ref{eq:control as unicycle}) guarantees 
	$\lim_{t\to\infty}(e_d(t),\beta(t),e_v(t),e_\omega(t)) = 0$,
	and the boundedness of all closed-loop signals.
\end{corollary}

The assumptions ${|\dot{p}_{\textup{d},1} \sin\phi - \dot{p}_{\textup{d,2}}\cos\phi |} < e_d\alpha_1$, ${\sin\beta} < e_d\alpha_2$ are imposed to avoid the singularity of $e_d=0$; note that $\beta$ and $\omega_\textup{d}$ are not defined in that case. Intuitively, they imply that $e_d$ will not be driven to zero faster than $\beta$ or $\dot{p}_{\textup{d},1} \sin\phi - \dot{p}_{\textup{d,2}}\cos\phi$; the latter becomes zero when the vehicle's velocity vector $v$ aligns with the desired one $\dot{p}_\textup{d}$. In the experiments, we tune the control gains according to $k_\beta \sim 10k_d$ in order to satisfy these assumptions.

\section{Numerical Experiments} \label{sec:exp res}

This section is devoted to a series of numerical experiments. More details can be found in Appendix \ref{app:B}.
We first test the proposed algorithm on a $6$-dof UR5 robotic manipulator with dynamics  
\begin{wrapfigure}[13]{r}{0pt}
	\centering
	\includegraphics[width=0.3\textwidth]{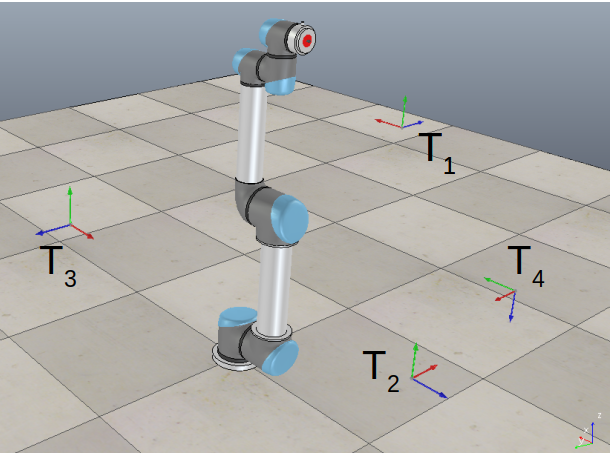}	
	\caption{A UR5 robot in a workspace with four points of interest $T_i$, $i\in\{1,\dots,4\}$.}
	\label{fig:ur5}
\end{wrapfigure}
\small
	\begin{align}\label{eq:exp dynamics}		
		\hspace{-2mm}\ddot{x} = B(x)^{-1}\left(u - C(\bar{x})\dot{x} - g(x) + d(\bar{x},t) \right)
	\end{align}
\normalsize
where $x, \dot{x} \in \mathbb{R}^6$ are the vectors of robot joint angles and angular velocities, respectively;
 $B(x)\in\mathbb{R}^{6\times 6}$ is the positive definite inertia matrix, $C(\bar{x})\in\mathbb{R}^{6\times 6}$ is the Coriolis matrix, $g(x)\in\mathbb{R}^6$ is the gravity vector, and $d(\bar{x},t)\in\mathbb{R}^6$ is a vector of friction terms and exogenous time-varying disturbances. 

The workspace consists of four points of interest $T_1$, $T_2$, $T_3$, $T_4$ 
 (end-effector position and Euler-angle orientation), as depicted in Fig. \ref{fig:ur5}, which correspond to the joint-angle vectors $c_1$, $c_2$, $c_3$, $c_4$. More information is provided in Appendix \ref{app:A}.
We consider a nominal SITL task of the form $\phi =  \bigwedge_{i\in\{1,\dots,4\}} G_{[0,\infty)} F_{I_i}
( \| x_1  -  c_i \| \leq 0.1) $, i.e., visit of $x_1$ to $c_i\in\mathbb{R}^6$ (within the radius $0.1$) infinitely often within the time intervals dictated by $I_i$, for $i\in\{1,\dots,4\}$.

We create 150 problem instances by varying the positions of $c_i$, the time intervals $I_i$, the dynamic parameters of the robot (masses and moments of inertia of the robot's links and actuators), the friction and disturbance term $d(\cdot)$, the initial position and velocity of the robot, and the sequence of visits to the points $c_i$, as dictated by $\phi$, i.e., one instance might correspond to the visit sequence $( (x(0),0) \to (c_1,t_{1_1}) \to (c_2,t_{1_2}) \to (c_3,t_{1_3}) \to (c_4,t_{1_4})$, and another to $( (x(0),0) \to (c_3,t_{1_3}) \to (c_1,t_{1_1}) \to (c_4,t_{1_4}) \to (c_2,t_{1_2})$. We separate the aforementioned 150 problem instances into 100 training instances and 50 test instances. We generate trajectories using the 100 training instances from system runs that satisfy different variations of one cycle of $\phi$ (i.e., one visit to each point). Each trajectory consists of 500 points and is generated using a nominal model-based controller. We use these trajectories to train a neural network and we test the control policy (\ref{eq:control as}) in the $50$ test instances. We also compare 
our algorithm with the non-adaptive controller $u_c(\bar{x},t) = u_\textup{nn}(x,t) - k_1 e - k_2 \dot{e}$, as well as with a modified version $u_d(\bar{x},t)$ of (\ref{eq:control as}) that does not employ the neural network (i.e., the term $u_\textup{nn}(\bar{x},t)$).
The comparison results are depicted in Fig. \ref{fig:errrors_comparison}, which depicts the mean and standard deviation of the signal $\|e(t)\|+\|\dot{e}(t)\|$ for the 
$50$ instances and 20 seconds.  
It is clear from the figure that the proposed algorithm performs better than the non-adaptive and no-neural-network policies both in terms of convergence speed and steady-state error. It is worth noting that the non-adaptive policy results on average in unstable closed-loop system. 

We next test the proposed algorithm, following a similar procedure, on a unicycle robot with dynamics of the form (\ref{eq:dynamics unicycle}). 
Fig. \ref{fig:unicycle errors} depicts the mean and standard deviation of the errors $e_d(t)$, $e_\beta(t)$ for 50 test instances. We note that the performance of the no-neural-network control policy is much more similar to the proposed one than in the UR5 case. This can be attributed to (1) the lack of gravitational terms in the unicycle dynamics, which often lead to instability, and (2) the chosen control gains of the no-neural-network policy, which are sufficiently large to counteract the effect of the dynamic uncertainties.

\begin{figure}
	\centering
	\includegraphics[trim={0cm 0cm 0cm 0cm},width=.7\textwidth]{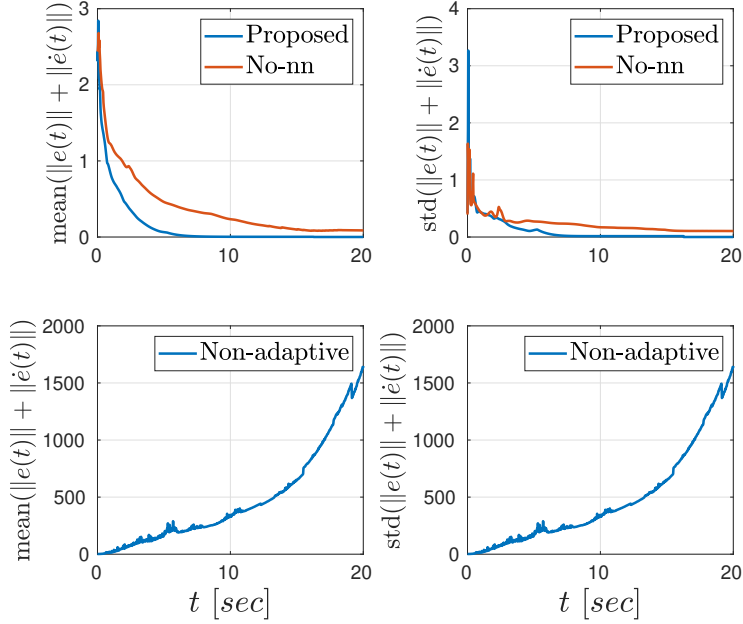}		
	\caption{Mean (left) and standard deviation (right) of $\|e(t)\|+\|\dot{e}(t)\|$ for the proposed, non-adaptive, and no-neural-network control policies.}
	\label{fig:errrors_comparison}
\end{figure}

Finally, we compare the performance of the proposed control policy with the reported data of (\cite{wang2019benchmarking}) on the benchmarking enivronment of the \textit{pendulum}, 
where a single-link mechanical structure aims to reach the upright position\footnote{For the sake of comparison, we do not consider an SITL task here.}. 
Following (\cite{wang2019benchmarking}), the reward and costs are $r_{pend}(t) = -\cos q(t) - 0.1\sin q(t) - 0.1\dot{q}(t) - 0.001u(t)^2$ and $J_{pend} = \sum_{t=1}^H \gamma^t r_{pend}(t)$, respectively. 
Similar to the previous cases and in contrast to \cite{wang2019benchmarking}, we generate 150 instances by varying the system parameters (pendulum length and mass), the external disturbances, and the initial conditions. 
\begin{wrapfigure}[18]{r}{0pt}
	\centering
	\includegraphics[trim={0cm 0cm 0cm 0cm},width=.4\textwidth]{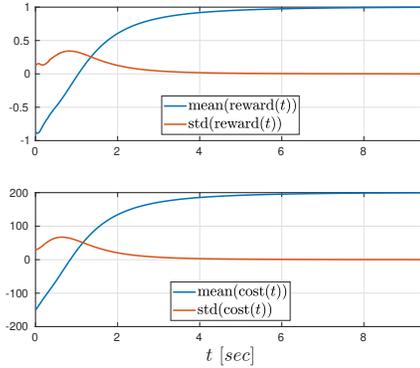}
	\caption{Mean and standard deviation of the reward and cost for the pendulum environment.}
	\label{fig:pendulum rew and costs}
\end{wrapfigure}
We generate 100 trajectories, consisting of 500 points each, for the 100 training instances, by employing a nominal controller, and we use them train a neural network. Moreover, in order to guarantee the feasibility of the proposed algorithm (\ref{eq:control as}) we consider larger control-action bounds than in (\cite{wang2019benchmarking}). The original bounds render these systems under-actuated, which is not included in the considered class of systems (\ref{eq:dynamics}) and consist part of our future work. It should be noted that such larger bounds affect negatively the acquired rewards. We test the control policy (\ref{eq:control as}) on the $50$ test instances, for 5000 steps each, corresponding to $10$ seconds. We set $H=200$ and $\gamma=1$ (\cite{wang2019benchmarking}). Fig. \ref{fig:pendulum rew and costs} depicts the mean and standard deviation of the time-varying reward and cost functions, illustrating successful regulation to the upright position. 
After 5000 steps, we obtain a mean reward and cost of $0.99$ and $200$, respectively, showing better performance than the reported cost of $180$ in (\cite{wang2019benchmarking}). Moreover, the proposed control algorithm achieves this performance without resorting to exhaustive exploration of the state-space, which is the case in the reinforcement-learning algorithms used in (\cite{wang2019benchmarking}). This is a very important property in practical engineering systems, where safety is of paramount significance and certain areas of the state space must be avoided. 
   
\begin{figure}
	\centering
	\includegraphics[trim={0cm 0cm 0cm 0cm},width=.8\textwidth]{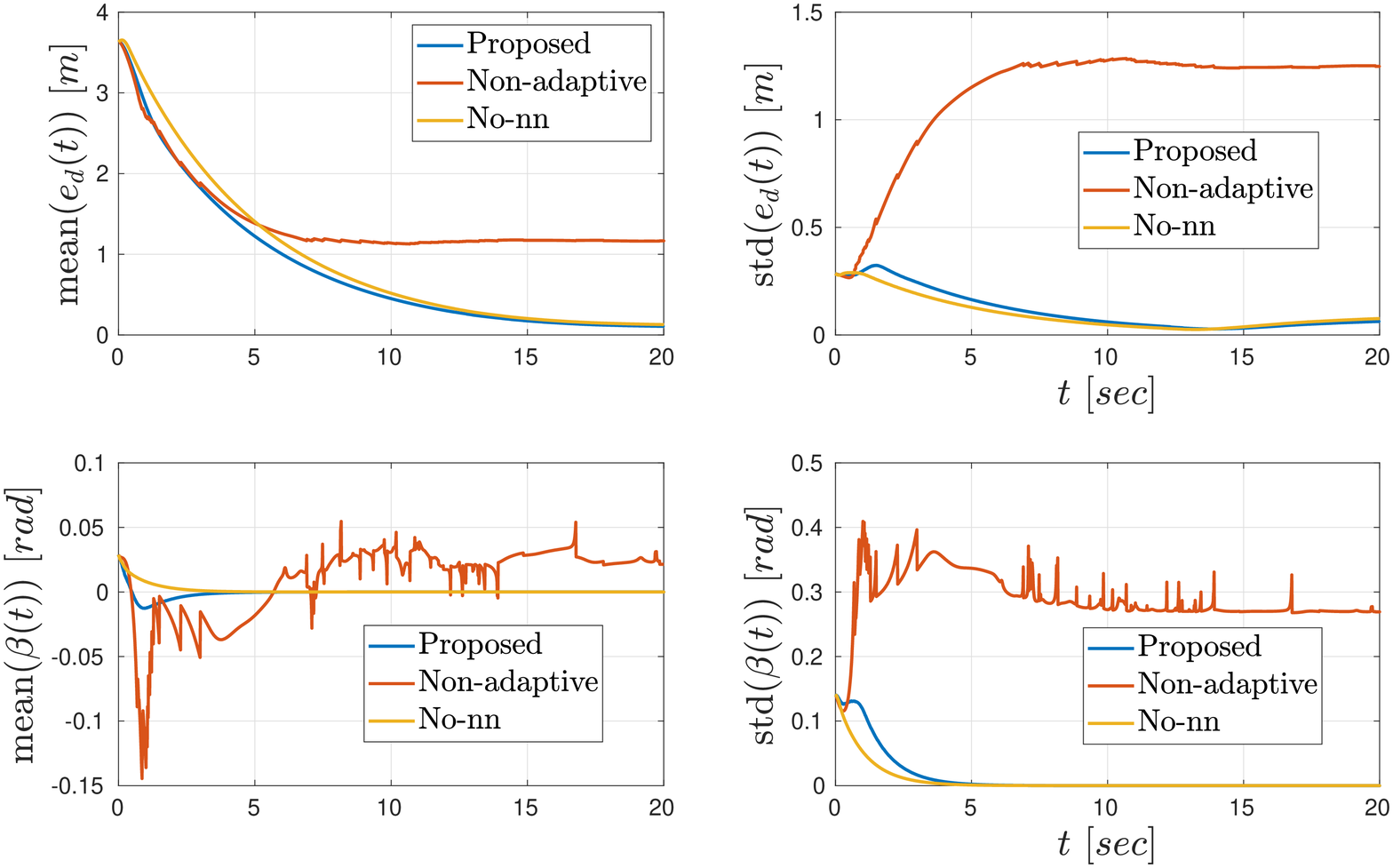}
	\caption{Mean (left) and standard deviation (right) of $e_d(t)$, $\beta(t)$ for the proposed, non-adaptive, and no-neural-network control policies.}
	\label{fig:unicycle errors}
\end{figure}
   


\section{Discussion and Limitations}
As shown in the experimental results, the control algorithm is able to accomplish tasks
which were not considered when generating the training data. Similarly, the trajectories used in the training data were generated using systems with different dynamical parameters, and not specifically the ones used in the tests. The aforementioned attributes signify the ability of the proposed algorithm to generalize to different tasks and systems with different parameters. Nevertheless, it should be noted that the control algorithm (5) guarantees asymptotic stability properties, without characterizing the
system’s transient state. Hence, task specifications that require fast system response (e.g., a robot needs to travel a long distance in a short time horizon) would possibly necessitate large control gains in order to achieve fast convergence to the derived open-loop trajectory. The
derivation of transient-state bounds and their relation with the given task specification (as for instance
in (\cite{verginis2018timed})) is left for future work. Moreover, the proposed control policy is currently limited by systems satisfying Assumption \ref{ass:g pd} and is not able to take into account under-actuated systems (e.g., the acrobot or cart-pole system); such systems consist part of our future work. Finally, the discontinuities of (\ref{eq:control as}), (\ref{eq:control as unicycle}) might be problematic and create
chattering when implemented in real actuators. A continuous approximation that has shown to yield
satisfying performance is the boundary-layer technique (\cite{slotine1991applied}).

\section{Conclusion and Future Work}

We develop a novel control algorithm for the control of robotic systems with unknown nonlinear
dynamics subject to task specifications expressed as SITL constraints. The algorithm integrates
neural network-based learning and adaptive control. We provide formal guarantees and perform
extensive numerical experiments. Future directions will focus on relaxing the considered assumptions and extending the proposed methodology
to underactuated systems.

\section{Reproducibility statement}
We provide the proofs of the theoretical results, i.e., Theorem \ref{th:as} and Corollary \ref{th:as unicycle}, in Appendix \ref{app:A}. Additionally, we elaborate on the required assumptions (Assumptions \ref{ass:g pd}, \ref{ass:bound}, \ref{ass:bound unicycle}) throughout the text - see pages 3, 5, and 7. 
Finally, we provide implementation details, instructions, and the respective code, required to reproduce the results, in Appendix \ref{app:B} and the supplementary material.

\bibliography{bibliography}
\bibliographystyle{iclr2022_conference}

\appendix
\section{Appendix} \label{app:A}

We provide here the proof of Theorem \ref{th:as} and Corollary \ref{th:as unicycle}. We first give some preliminary notation and background on systems with discontinuous dynamics.

\subsection*{Notation}
Given a function $f:\mathbb{R}^n \to \mathbb{R}^k$, its {Filippov regularization is defined as~(\cite{paden1987calculus})}
\begin{equation} \label{eq:Filippov regular.}
	\mathsf{K}[f](x) \coloneqq \bigcap_{\delta > 0}\bigcap_{\mu(\bar{N})=0} \overline{\text{co}}(f(\mathcal{B}(x,\delta) \backslash \bar{N}),t),
\end{equation}
where $\bigcap_{\mu(\bar{N})=0}$ is the intersection over all sets $\bar{N}$ of Lebesgue measure zero, $\overline{\text{co}}(E)$ is the closure of the convex hull $\text{co}(E)$ of the set $E$,  and $\mathcal{B}(x,\delta)$ is the ball with radius $\delta$ centered at $x$. 

\subsection*{Nonsmooth Analysis} \label{subsec:Nonsmooth}

Consider the following differential equation with a discontinuous
right-hand side:
\begin{equation} \label{eq:discont diff eq}
	\dot{x} = f(x,t),
\end{equation}
where $f:\mathcal{D} \times [t_0,\infty) \to \mathbb{R}^n$, $\mathcal{D}\subset \mathbb{R}^n$, is Lebesgue measurable and locally essentially bounded.
\begin{definition}[Def. $1$ of~(\cite{fischer2013lasalle})] \label{def:Filipp sol}	
	A function $x:[t_0,t_1)\to\mathbb{R}^n$, with $t_1 > t_0$, is called a Filippov solution of (\ref{eq:discont diff eq}) on $[t_0,t_1)$ if $x(t)$ is absolutely continuous and if, for almost all $t\in[t_0,t_1)$, it satisfies $\dot{x} \in \mathsf{K}[f](x,t)$, where $\mathsf{K}[f](x,t)$ is the Filippov regularization of $f(x,t)$. 
\end{definition}
\begin{lemma}[Lemma $1$ of~(\cite{fischer2013lasalle})] \label{lem:Chain rule}
	Let $x(t)$ be a Filippov solution of (\ref{eq:discont diff eq}) and $V:\mathcal{D}\times [t_0,t_1)\to\mathbb{R}$ be a locally Lipschitz, regular function\footnote{See~(\cite{fischer2013lasalle}) for a definition of regular functions.}. Then $V(x(t),t)$ is absolutely continuous, $\dot{V}(x(t),t)=\frac{\partial}{\partial t}V(x(t),t)$ exists almost everywhere (a.e.), i.e., for almost all $t\in [t_0,t_1)$, and $\dot{V}(x(t),t) \overset{\text{a.e}}{\in} \dot{\widetilde{V}}(x(t),t)$, where
	\begin{equation*}
		\dot{\widetilde{V}} \coloneqq \bigcap_{\xi\in\partial V(x,t)} \xi^\top \begin{bmatrix}
			\mathsf{K}[f](x,t) \\ 1
		\end{bmatrix},
	\end{equation*} and $\partial V(x,t)$ is Clarke's generalized gradient at $(x,t)$~(\cite{fischer2013lasalle}). 
\end{lemma}
\begin{corollary}[Corollary $2$ of~(\cite{fischer2013lasalle})] \label{th:LaSalle}
	For the system given in (\ref{eq:discont diff eq}), let $\mathcal{D}\subset\mathbb{R}^n$ be an open and connected set containing $x=0$ and suppose that $f$ is Lebesgue measurable and $x\mapsto f(x,t)$ is essentially locally bounded, uniformly in $t$. Let $V:\mathcal{D}\times [t_0,t_1) \to\mathbb{R}$ be locally Lipschitz and regular such that $W_1(x) \leq V(x,t) \leq W_2(x)$, $\forall t\in [t_0,t_1)$, $x\in\mathcal{D}$, and
	\begin{equation*}
		z \leq -W(x(t)), \ \ \forall z\in \dot{\widetilde{V}}(x(t),t), \ t\in [t_0,t_1), \ x \in \mathcal{D},
	\end{equation*}
	where $W_1$ and $W_2$ are continuous positive definite functions and $W$ is a continuous positive semi-definite  on $\mathcal{D}$. Choose $r>0$ and $c>0$ such that $\bar{\mathcal{B}}(0,r) \subset \mathcal{D}$ and $c<\min_{\|x\|=r} W_1(x)$. Then for all Filippov solutions $x:[t_0,t_1)\to\mathbb{R}^n$ of (\ref{eq:discont diff eq}), with $x(t_0)\in {\mathbb{D}} \coloneqq\{x\in \bar{\mathcal{B}}(0,r) : W_2(x) \leq c\}$, it holds that $t_1 = \infty$, $x(t)\in {\mathbb{D}}$, $\forall t\in[t_0,\infty)$, and  $\lim_{t\to\infty} W(x(t)) = 0$.
\end{corollary}

\begin{proof}[Proof of Theorem \ref{th:as}]
	
		We re-write first the condition of Assumption \ref{eq:assumption bound}. Note that 
		$p_{\textup{d}}(t)$ and its derivatives are bounded functions of time. Moreover, since $e=x-p_\textup{d}$, $\dot{e} = \dot{x} - \dot{p}_\textup{d}$, it holds that $\|\bar{x}\| \leq \|e\| + \|\dot{e}\| + \|p_\textup{d}\| + \|\dot{p}_\textup{d}\|$ and $\dot{v}_\textup{d} = \ddot{p}_\textup{d} - k_1\dot{e}$, as well as $e_v = \dot{x} - v_\textup{d} = \dot{e} + k_1 e$ implying $\dot{e} = e_v - k_1 e$. Therefore, in view of (\ref{eq:assumption bound}), one can find positive constants $d_1$, $d_2$, $D$ such that 
		 \begin{align}\label{eq:assumption bound new}
		 	\frac{1}{\underline{g}}\|\underline{g} e + f(\bar{x},t) + g(\bar{x},t)u_\textup{nn}(\bar{x},t) - \dot{v}_\textup{d} \| \leq d_1 \|e \| + d_2\|e_v\| + D,
		 \end{align}
		 for all $\bar{x}\in\mathbb{R}^{2n}$, $t \geq 0$, where $\underline{g}$ is the minimum eigenvalue of $g(\bar{x},t)$, which is positive owing to the positive definitiveness of $g(\cdot)$. Inequality (\ref{eq:assumption bound new}) will be used later in the proof. 
		 
		 Let now a constant $\alpha$ such that $\frac{d_1 \alpha}{2} < k_1$. As will be clarified later, the adaptation variables $\hat{\ell}_1$, $\hat{\ell}_2$ aim to approximate the constants $\frac{d_1}{2\alpha} + d_2$ and $D$, respectively. 	
		 Therefore, let $\ell_1 \coloneqq \frac{d_1}{2\alpha} + d_2$, $\ell_2 \coloneqq D$, and the respective error terms $\widetilde{\ell}_1 \coloneqq \hat{\ell}_1 - \ell_1$, $\widetilde{\ell}_2 \coloneqq \hat{\ell}_2 - \ell_2$, as well as the overall state $\widetilde{x} \coloneqq [e^\top,e_v^\top,\widetilde{\ell}_1,\widetilde{\ell}_2]^\top \in \mathbb{R}^{{2n}+2}$. 	
	 	Since the control policy is discontinuous, we use the notion of Filippov solutions. The Filippov regularization of $u$ is $\mathsf{K}[u] = u_\textup{nn}(x,t) - k_2e_v - \hat{\ell}_1e_v -  \hat{\ell}_2\hat{\mathsf{E}}_v$, where $\hat{\mathsf{E}}_v \coloneqq \frac{e_v}{\|e_v\|}$ if $e_v\neq 0$ and $\hat{\mathsf{E}}_v \in (-1,1)$ otherwise. Note that, in any case, it holds that $e_v^\top \hat{\mathsf{E}}_v = \|e_v\|$. 
		
		Let now the continuously differentiable function 
		\begin{align*}
			V(\widetilde{x})  \coloneqq \frac{1}{2}\|e\|^2 + \frac{1}{2\underline{g}}\|e_v\|^2 + \sum_{i\in\{1,2\}}\frac{1}{2k_{\ell_i}}\widetilde{\ell}_i 
		\end{align*}	
		which satisfies $W_1(\widetilde{x}) \leq V(\widetilde{x}) \leq W_2(\widetilde{x})$ for $W_1(\widetilde{x}) \coloneqq \min\{\frac{1}{2}, \frac{1}{2\underline{g}}, \frac{1}{2k_{\ell_1}}, \frac{1}{2k_{\ell_2}}\}\|\widetilde{x}\|^2$, $W_2(\widetilde{x}) \coloneqq \max\{\frac{1}{2}, \frac{1}{2\underline{g}}, \frac{1}{2k_{\ell_1}}, \frac{1}{2k_{\ell_2}}\}\| \widetilde{x} \|^2$. 
		According to Lemma \ref{lem:Chain rule}, $\dot{V}(\widetilde{x}) \overset{a.e.}{\in} \dot{\widetilde{V}}(\widetilde{x})$, with $\dot{\widetilde{V}} \coloneqq \bigcap_{\xi \in \partial V(\widetilde{x})} \mathsf{K}[\dot{\widetilde{x}}]$. Since $V(\widetilde{x})$ is continuously differentiable, its generalized gradient reduces to the standard gradient and thus it holds that $\dot{\widetilde{V}}(\widetilde{x}) = \nabla V^\top \mathsf{K}[\dot{\widetilde{x}}]$, where $\nabla V = [e^\top, \frac{1}{\underline{g}} e_v^\top, \frac{1}{k_{\ell_1}}\widetilde{\ell}_1, \frac{1}{k_{\ell_2}}\widetilde{\ell}_2]^\top$.
		By differentiating $V$ and using (3), one obtains	
		\begin{align*}
			\dot{V} \subset& \widetilde{W}_s \coloneqq e^\top(\dot{x} - \dot{p}_\textup{d}) +  \frac{1}{\underline{g}}e_v^\top \big(  
			 f(\bar{x},t) + g(\bar{x},t)u - \dot{v}_{\textup{d}}
			\big) + \frac{1}{k_{\ell_1}}\widetilde{\ell}_1\dot{\hat{\ell}}_1 + \frac{1}{k_{\ell_2}}\widetilde{\ell}_2\dot{\hat{\ell}}_2
		\end{align*}
		and by using $\dot{x} = e_v + v_\textup{d}$ and substituting the control policy (\ref{eq:v d}), (\ref{eq:control as}), and inequality (\ref{eq:assumption bound new}), 
		\begin{align*}
			\widetilde{W}_s \coloneqq& -k_1\|e\|^2 + \frac{1}{\underline{g}}e_v^\top\big(\underline{g}e + f(\bar{x},t) + g(\bar{x},t)u_\textup{nn}(\bar{x},t) - \dot{v}_\textup{d}  \big) - \frac{1}{\underline{g}}e_v^\top g(\bar{x},t)\big( k_2 e_v + \hat{\ell}_1e_v + \hat{\ell}_2 \hat{\mathsf{E}}_v  \big) \\						
			& + \widetilde{\ell}_1 \|e_v\|^2 + \widetilde{\ell}_2 \|e_v\| \\
			\leq& -k_1\|e\|^2 + d_1\|e_v\|\|e\| + d_2\|e_v\|^2 + D\|e_v\| - \frac{1}{\underline{g}}e_v^\top g(\bar{x},t)\big( k_2 e_v + \hat{\ell}_1e_v + \hat{\ell}_2 \hat{\mathsf{E}}_v  \big) \\
			&+ \widetilde{\ell}_1 \|e_v\|^2 + \widetilde{\ell}_2 \|e_v\|
		\end{align*}
	Note from (\ref{eq:control as}) and the fact that $\hat{\ell}_1(0) > 0$, $\hat{\ell}_2(0) > 0$ that $\hat{\ell}_1(t) > 0$ and $\hat{\ell}_2(t)) > 0$ for all $t\geq 0$. Moreover, recall that $g(\bar{x},t)$ is positive definite for all $\bar{x}\in\mathbb{R}^{2n}$, $t\geq0$, with $\underline{g}$ being its minimum (and positive) eigenvalue.
	 Therefore, we conclude that $\widetilde{W}_s$ becomes
	 \small
	 \begin{align*}
	 	\widetilde{W}_s  \leq& -k_1\|e\|^2 + d_1\|e_v\|\|e\| + d_2\|e_v\|^2 + D\|e_v\| - (k_2+\hat{\ell}_1)\|e_v\|^2 - \hat{\ell}_2\|e_v\| + \widetilde{\ell}_1 \|e_v\|^2 + \widetilde{\ell}_2 \|e_v\|
	 \end{align*}
 \normalsize
	By incorporating the term $\alpha$ in the term $d_1\|e_v\|\|e\|$, we obtain $d_1\|e_v\|\|e\| = d_1\alpha \frac{\|e_v\|}{\alpha}\|e\|$ and by using the property $ab \leq \frac{1}{2}a^2 + \frac{1}{2}b^2$ for any constants $a$,$b$, we obtain $d_1\alpha \frac{\|e_v|}{\alpha}\|e\| \leq \frac{d_1\alpha}{2}\|e\|^2 + \frac{d_1}{2\alpha}\|e_v\|^2$. Therefore, $\widetilde{W}_s$ becomes
	\small
	\begin{align*}
	\widetilde{W}_s \leq & -\left(k_1 - \frac{d_1\alpha}{2}\right)\|e\|^2 + \left(\frac{d_1}{2\alpha} + d_2\right)\|e_v\|^2 + D\|e_v\| - (k_2+\hat{\ell}_1)\|e_v\|^2 - \hat{\ell}_2\|e_v\| + \widetilde{\ell}_1 \|e_v\|^2+ \widetilde{\ell}_2 \|e_v\| \\
	=& -\left(k_1 - \frac{d_1\alpha}{2}\right)\|e\|^2 + \ell_1 \|e_v\|^2 + \ell_2 \|e_v\| - (k_2+\hat{\ell}_1)\|e_v\|^2 - \hat{\ell}_2\|e_v\|+ \widetilde{\ell}_1 \|e_v\|^2+ \widetilde{\ell}_2 \|e_v\| \\
	=& -\left(k_1 - \frac{d_1\alpha}{2}\right)\|e\|^2 - k_2\|e_v\|^2 =:-Q(\widetilde{x})
	\end{align*}
\normalsize
	Therefore, it holds that $\zeta \leq - Q(\widetilde{x})$, for all $\zeta\in \dot{\widetilde{V}}$, with $Q$ being a continuous and positive semi-definite function in $\mathbb{R}^{2n+2}$, since $k_1 - \frac{d_1\alpha}{2}>0$. Choose now any finite $r>0$ and let $c < \min_{\|\widetilde{x}\|=r} W_1(\widetilde{x})$. Hence, all the conditions of Corollary \ref{th:LaSalle} are satisfied and hence, all Filippov solutions starting from $\widetilde{x}(0) \in \Omega_f \coloneqq \{ \widetilde{x} \in \mathcal{B}(0,r) : \widetilde{W}_2(\widetilde{x}) \leq c \}$ are bounded and remain in $\Omega_f$, satisfying $\lim_{t\to\infty} Q(\widetilde{x}(t)) = 0$. 	
		Note that $r$, and hence $c$, can be arbitrarily large allowing and finite initial condition $\widetilde{x}(0)$. Moreover, the boundedness of $\widetilde{x}$ implies the boundedness of $e$, $\dot{e}$, $e_v$,  $\widetilde{\ell}_1$, and $\widetilde{\ell}_2$, and hence of $\hat{\ell}_1(t)$ and $\hat{\ell}_2(t)$, for all $t\in\mathbb{R}_{\geq 0}$. In view of (5), we finally conclude the boundendess of $u(\cdot)$, $\dot{\hat{\ell}}_1$, and $\dot{\hat{\ell}}_2$, for all $t\in\mathbb{R}_{\geq 0}$, leading to the conclusion of the proof.
		
\end{proof}

\begin{proof}[Proof of Corollary \ref{th:as unicycle}]	
	
	The derivatives of $e_d$, $\beta$ in (\ref{eq:ed beta dot}) can be written as 
	\begin{align*}
		\begin{bmatrix}
			\dot{e}_d \\
			\dot{\beta}
		\end{bmatrix}
		= 
		G\begin{bmatrix}
			v \\ 
			\omega
		\end{bmatrix}
		+ \begin{bmatrix}
			\dot{p}_{\textup{d},1}\cos(\phi+\beta) + \dot{p}_{\textup{d},2}\sin(\phi+\beta) \\
			 \frac{\dot{p}_{\textup{d},2}}{e_d}\cos(\phi+\beta) - \frac{\dot{p}_{\textup{d},1}}{e_d}\sin(\phi+\beta) 
		\end{bmatrix}
	\end{align*}
	where 
	\begin{align*}
		G \coloneqq \begin{bmatrix}
			-\cos\beta & 0 \\
			\frac{\sin\beta}{e_d} & - 1
		\end{bmatrix}
	\end{align*}

	It can be verified that the reference velocity signals, designed in (\ref{eq:v_des unicycle}), can be written as 
	\begin{align*}
		\begin{bmatrix}
			v_{\textup{d}} \\
			\omega_{\textup{d}}
		\end{bmatrix}		
	= G^{-1}\begin{bmatrix}
		-\dot{p}_{\textup{d},1}\cos(\phi+\beta) - \dot{p}_{\textup{d},2}\sin(\phi+\beta) - k_de_d \\
		- \frac{\dot{p}_{\textup{d},2}}{e_d}\cos(\phi+\beta) + \frac{\dot{p}_{\textup{d},1}}{e_d}\sin(\phi+\beta) - k_\beta \beta
	\end{bmatrix}
	\end{align*}
	and therefore, by using the relations $v = e_v + v_{\textup{d}}$ and $\omega = e_\omega + \omega_{\textup{d}}$, $\dot{e}_d$ and $\dot{\beta}$ can be written as 
		\begin{align} \label{eq:ed beta dot 2}	
			\begin{bmatrix}
				\dot{e}_d \\
				\dot{\beta}
			\end{bmatrix} 
			= \begin{bmatrix}
				-k_d e_d \\
				-k_\beta \beta
			\end{bmatrix}
			+ G \begin{bmatrix}
				e_v \\
				e_\omega
			\end{bmatrix} = 
			\begin{bmatrix}
				-k_d e_d \\
				-k_\beta \beta
			\end{bmatrix}
			+ \begin{bmatrix}
				-e_v \cos\beta \\
				e_v \frac{\sin\beta}{e_d} - e_\omega
			\end{bmatrix}		
		\end{align}
	

	The control design follows the back-stepping methodology (\cite{krstic1995nonlinear}). Let, therefore, the function 
	$V_1 \coloneqq \frac{1}{2}(e_d^2 + \beta^2)$, which, after time differentiation and use of (\ref{eq:ed beta dot 2}), yields
	\begin{align} \label{eq:V_1_dot uni}
		\dot{V}_1 =& -k_d e_d^2 - k_\beta \beta^2 - e_de_v\cos\beta + \beta e_v \frac{\sin\beta}{e_d} - \beta e_\omega
	\end{align}
	We will cancel the two last terms of $\dot{V}_1$ using the control input (\ref{eq:control as unicycle}). 
	
	First, we note that the inertia matrix of the system $M$, appearing in the unicycle dynamics (\ref{eq:dynamics unicycle}), has the form (\cite{ivanjko2010modelling}) $M = \begin{bmatrix}
		M_1 & M_2 \\
		M_2 & M_1 
	\end{bmatrix}$
	with $M_1 \coloneqq \frac{mr^2}{4} + \frac{(I_C + md^2)r^2}{4R^2} + I_0$, $M_2 \coloneqq \frac{mr^2}{4} - \frac{(I_C + md^2)r^2}{4R^2}$, and where $I_C$ is the moment of inertia of the vehicle with respect to point C (see Fig. \ref{fig:unicycle}), $I_0$ is the moment of inertia of the the wheels, and $d$ is the distance of the between point C and the vehicle's center of mass $(p_1,p_2)$. Therefore, the second part of the unicycle dynamics (\ref{eq:dynamics unicycle}) can be written as 
	\begin{align*}
		&M_1 \ddot{\theta}_R + M_2 \ddot{\theta}_L = u_R  + f_{\theta,R}(\bar{x},t) \\
		&M_2 \ddot{\theta}_R + M_1 \ddot{\theta}_L = u_L  + f_{\theta,L}(\bar{x},t) \\
	\end{align*}
	with $f_{\theta,R}$ and $f_{\theta,L}$ denoting the elements of $f_\theta$. By summing and subtracting the aforementioned equations, we obtain
	\begin{subequations}	\label{eq:unic dynamics sum and diff}
		\begin{align} 
			&(M_1 + M_2)(\ddot{\theta}_R + \ddot{\theta}_L) = u_R + u_L + f_{\theta,R}(\bar{x},t) + f_{\theta,L}(\bar{x},t) \\
			&(M_1 - M_2)(\ddot{\theta}_R - \ddot{\theta}_L) = u_R - u_L + f_{\theta,R}(\bar{x},t) - f_{\theta,L}(\bar{x},t)
		\end{align}
	\end{subequations}
	From the definition of $e_v$ and $e_\omega$, it holds that $\dot{e}_v = \dot{v}-\dot{v}_{\textup{d}}$, $\dot{e}_\omega = \dot{\omega}-\dot{\omega}_{\textup{d}}$, which, by using the relations $v = \frac{r}{2}(\dot{\theta}_R + \dot{\theta}_L)$, $\omega = \frac{r}{2R}(\dot{\theta}_R - \dot{\theta}_L)$ and (\ref{eq:unic dynamics sum and diff}), becomes 
	\begin{subequations}	\label{eq:unic dynamics sum and diff 2}
		\begin{align} 
			&\dot{e}_v = \frac{r}{2(M_1+M_2)}\bigg(u_R + u_L + f_{\theta,R}(\bar{x},t) + f_{\theta,L}(\bar{x},t) \bigg)  - \dot{v}_{\textup{d}}\\
			&\dot{e}_\omega = \frac{r}{2R(M_1 - M_2)}\bigg(u_R - u_L + f_{\theta,R}(\bar{x},t) - f_{\theta,L}(\bar{x},t)  \bigg) - \dot{\omega}_{\textup{d}}
		\end{align}
	\end{subequations}		
	
	Define now $\ell_{v} \coloneqq 2\frac{M_1+M_2}{r}$, $\ell_{\omega} \coloneqq 2R\frac{M_1 - M_2}{r}$. The adaptation terms $\hat{\ell}_{v}$, and $\hat{\ell}_{\omega}$, used in the control mechanism (\ref{eq:control as unicycle}), aim to approximate $\ell_{v}$ and $\ell_{\omega}$, respectively. Note that $\hat{\ell}_{v}$ and $\hat{\ell}_{\omega}$ are positive, and we define the errors $\widetilde{\ell}_v \coloneqq \hat{\ell}_v - {\ell}_v$, $\widetilde{\ell}_\omega \coloneqq \hat{\ell}_\omega - \ell_\omega$. 
	
	We re-write next the condition of Assumption \ref{ass:bound unicycle}. It holds that 
	$\|p\| = \sqrt{p_1^2+p_2^2} \leq 2e_d + |p_{\textup{d},1}| + |p_{\textup{d},2}|$. Moreover, in view of (\ref{eq:v_des unicycle}) as well as the fact that $\beta\in(-\bar{\beta},\bar{\beta})\subset(-\frac{\pi}{2},\frac{\pi}{2})$, where $\bar{\beta}$ is the bound of $\beta$, stated in Theorem \ref{th:as unicycle}, it holds that $|v| \leq |e_v| + |v_\textup{d}| \leq |e_v| + \frac{1}{\cos(\bar{\beta})}(|\dot{p}_{\textup{d},1}| + |\dot{p}_{\textup{d},2}| + k_de_d)$, $|\omega| \leq |e_\omega| + |\omega_\textup{d}| \leq |e_\omega| + \frac{1}{\cos\bar{\beta}}( \alpha_1 + k_d) + k_\beta |\beta|$, where we also use the fact that $\frac{|\dot{p}_{\textup{d},1} \sin\phi - \dot{p}_{\textup{d},2}\cos\phi|}{e_d} \leq \alpha_1$. By also recalling that $\phi$ moves on the unit circle, we conclude that Assumption \ref{ass:bound unicycle} becomes
	\begin{align} \label{eq:assumption bound unicycle new}
		\|f_{\theta}(\bar{x},t) + u_\textup{nn}(\bar{x},t)\| \leq d_d e_d + d_\beta |\beta| + d|e_v| + d|e_\omega| + D,
	\end{align}
	with $d_d \coloneqq d(2 + \frac{k_d}{\cos\bar{\beta}})$, $d_\beta \coloneqq dk_\beta$, and $D\coloneqq d(2\bar{p}_{\textup{d}} + 2\pi + \frac{1}{\cos\bar{\beta}}(\alpha_1 + 2\bar{v}_{\textup{d}}+ 1)$, where $\bar{p}_{\textup{d}}$ and $\bar{v}_{\textup{d}}$ are the upper bounds of $p_{\textup{d},i}$, $v_{\textup{d},i}$, respectively, $i\in\{1,2\}$.

	Let now positive constants $\gamma_d$ and $\gamma_\beta$ such that $k_d > 2d_d\gamma_d$ and $k_\beta > 2d_\beta \gamma_\beta$, and define 
	 \begin{align}
	 	\ell_{1} \coloneqq& \frac{d_d}{\gamma_d} + \frac{d_\beta}{\gamma_\beta} + 4d 
 	 \end{align}
	 which is the constant to be approximated by $\hat{\ell}_{1}$; its derivation will be clarified later. Let also $\ell_2 \coloneqq {2D}$, which we aim to approximate with the adaptation variable $\hat{\ell}_2$. 
	 We define hence the respective errors $\widetilde{\ell}_1 \coloneqq \hat{\ell}_{1} - \ell_{1}$ and $\widetilde{\ell}_2 \coloneqq \hat{\ell}_2 - \ell_2$ and define the overall vector $\widetilde{x} \coloneqq [e_d,\beta,e_v,e_\omega,\widetilde{\ell}_v,\widetilde{\ell}_\omega,\widetilde{\ell}_1,\widetilde{\ell}_2]^\top \in \mathbb{R}^8$. 
	 Since the control mechanism is discontinuous, we use again the notion of Filippov solutions. The Filippov regularization $\mathsf{K}[u]$ of $u$ differs from $u$ only in the terms $\hat{e}_v$ and $\hat{e}_\omega$, which are replaced by $\hat{\mathsf{E}}_v$, defined as $\hat{\mathsf{E}}_v \coloneqq \frac{e_v}{|e_v|}$ if $e_v\neq 0$, and $\hat{\mathsf{E}}_v \in (-1,1)$ otherwise, and $\hat{\mathsf{E}}_\omega$, defined as $\hat{\mathsf{E}}_\omega \coloneqq \frac{e_\omega}{|e_\omega|}$ if $e_\omega\neq 0$, and $\hat{\mathsf{E}}_\omega \in (-1,1)$, respectively.	
	
	  Consider now the continuously differentiable and positive definite function 
	 \begin{align*}
	 	V(\bar{x}) \coloneqq V_1 + \frac{\ell_v}{2} e_v^2 + \frac{\ell_\omega}{2} e_\omega^2  + \frac{1}{2 k_{v}}\widetilde{\ell}_{v}^2 +  \frac{1}{2 k_{\omega}}\widetilde{\ell}_{\omega_i}^2 + \sum_{i\in\{1,2\}}\frac{1}{2 k_{i}}\widetilde{\ell}_{i}^2 
	 \end{align*}
	 which satisfies $W_1(\widetilde{x}) \leq V(\widetilde{x}) \leq W_2(\widetilde{x})$ for $W_1(\widetilde{x}) \coloneqq \min\{ \frac{1}{2}, \frac{\ell_v}{2}, \frac{\ell_\omega}{2}, \frac{1}{2k_{v}}, \frac{1}{2k_{\omega}}, \frac{1}{2k_1}, \frac{1}{2k_2}\}\|\widetilde{x}\|^2$, $W_2(\widetilde{x}) \coloneqq \max\{ \frac{1}{2}, \frac{\ell_v}{2}, \frac{\ell_\omega}{2}, \frac{1}{2k_{v}}, \frac{1}{2k_{\omega}}, \frac{1}{2k_1}, \frac{1}{2k_2}\}\| \widetilde{x} \|^2$.
	 According to Lemma \ref{lem:Chain rule}, $\dot{V}(\widetilde{x}) \overset{a.e.}{\in} \dot{\widetilde{V}}(\widetilde{x})$, with $\dot{\widetilde{V}} \coloneqq \bigcap_{\xi \in \partial V(\widetilde{x})} \mathsf{K}[\dot{\widetilde{x}}]$. Since $V(\widetilde{x})$ is continuously differentiable, its generalized gradient reduces to the standard gradient and thus it holds that $\dot{\widetilde{V}}(\widetilde{x}) = \nabla V^\top \mathsf{K}[\widetilde{x}]$, where $\nabla V = [e_d, \beta, \ell_ve_v, \ell_\omega e_\omega, \frac{1}{k_v}\widetilde{\ell}_v, \frac{1}{k_{\omega}}\widetilde{\ell}_\omega, \frac{1}{k_1}\widetilde{\ell}_1, \frac{1}{k_2}\widetilde{\ell}_2]^\top$.	 
	 Therefore, by differentiating $V$ and employing (\ref{eq:V_1_dot uni}) and (\ref{eq:unic dynamics sum and diff}), we obtain
	 \begin{align*}
	 	\dot{V} \subset \widetilde{W} \coloneqq & -k_d e_d^2 - k_\beta \beta^2 - e_de_v\cos\beta + \beta e_v \frac{\sin\beta}{e_d} - \beta e_\omega + e_v\big(u_R + u_L + f_{\theta,R} + f_{\theta,L} - \ell_{v}\dot{v}_{\textup{d}} \big) \\
	 	&+ e_\omega\big(u_R - u_L + f_{\theta,R} - f_{\theta,L} - \ell_{\omega}\dot{\omega}_{\textup{d}} \big) + \frac{1}{k_v}\widetilde{\ell}_v\dot{\hat{\ell}}_v + \frac{1}{k_\omega}\widetilde{\ell}_\omega\dot{\hat{\ell}}_\omega + \sum_{i\in\{1,2\}} \frac{1}{k_i}\widetilde{\ell}_i \dot{\hat{\ell}}_i
	 \end{align*}
	 By substituting the control and adaptation policy (\ref{eq:control as unicycle}), $\widetilde{W}$ becomes
	 \begin{align*}
	 	\widetilde{W} =&  -k_d e_d^2 - k_\beta \beta^2 - k_v e_v^2 - k_\beta e_\omega^2 + e_v \dot{v}_{\textup{d}}\widetilde{\ell}_{v} + e_\omega \dot{\omega}_{\textup{d}}\widetilde{\ell}_{\omega} - \hat{\ell}_{1}(e_v^2 + e_\omega^2) - \hat{\ell}_2(|e_v|+|e_\omega|) \\
	 	& + e_v(u_{\textup{nn},R} + u_{\textup{nn},L} + f_{\theta,R} + f_{\theta,L})+ e_\omega(u_{\textup{nn},R} - u_{\textup{nn},L} + f_{\theta,R} - f_{\theta,L}) \\
	 	& - \widetilde{\ell}_ve_v\dot{v}_\textup{d} - \widetilde{\ell}_\omega e_\omega\dot{\omega}_\textup{d} + \widetilde{\ell}_1 (e_v^2 + e_\omega^2)  + \widetilde{\ell}_2 (|e_v| + |e_\omega|)
	 \end{align*}	 
	 where $u_{R,\textup{nn}}$ and $u_{L,\textup{nn}}$ are the elements of $u_\textup{nn}$, i.e., $u_\textup{nn} = [u_{R,\textup{nn}},u_{L,\textup{nn}}]^\top$.	 
	 It is easy to conclude that $|u_{\textup{nn},R} + u_{\textup{nn},L} + f_{\theta,R} + f_{\theta,L}| \leq |u_{\textup{nn},R} + f_{\theta,R}| + |u_{\textup{nn},L} + f_{\theta,L}| \leq 2\|u_{\textup{nn}} + f_\theta\|$ and $|u_{\textup{nn},R} - u_{\textup{nn},L} + f_{\theta,R} - f_{\theta,L}| \leq |u_{\textup{nn},R} + f_{\theta,R}| + |u_{\textup{nn},L} + f_{\theta,L}| \leq 2\|u_{\textup{nn}} + f_\theta\|$ and hence, in view of (\ref{eq:assumption bound unicycle new}), 
	 \begin{align*}
		\left.
	 	\begin{matrix}
	 		|u_{\textup{nn},R} + u_{\textup{nn},L} + f_{\theta,R} + f_{\theta,L}|  \\
	 		|u_{\textup{nn},R} - u_{\textup{nn},L} + f_{\theta,R} - f_{\theta,L}| 		
	 	\end{matrix} \right\} 		
	 	\leq \bar{D} \coloneqq 2d_d e_d + 2d_\beta |\beta| + 2d|e_v| + 2d|e_\omega| + 2D 
	 \end{align*}
	Therefore, $\widetilde{W}$ becomes
	\begin{align*}
		\widetilde{W} =&  -k_d e_d^2 - k_\beta \beta^2 - k_v e_v^2 - k_\beta e_\omega^2  - \hat{\ell}_{1}(e_v^2 + e_\omega^2) - \hat{\ell}_2(|e_v|+|e_\omega|) + \bar{D}(|e_v|+|e_\omega|) \\
		& + \widetilde{\ell}_1 (e_v^2 + e_\omega^2)  + \widetilde{\ell}_2 (|e_v| + |e_\omega|)
	\end{align*}
	The term $\bar{D}(|e_v|+|e_\omega|)$ yields 
	\begin{align*}
		\bar{D}(|e_v|+|e_\omega|) =& 2d_de_d|e_v| + 2d_\beta|\beta||e_v| + 2de_v^2 + 4d|e_v||e_\omega| + 2D|e_v| + 2d_de_d|e_\omega|  \\
		& + 2d_\beta|\beta||e_\omega| + 2de_\omega^2 + 2D|e_\omega|
	\end{align*}
	By setting $2d_de_d|e_v| = 2d_d\gamma_de_d\frac{|e_v|}{\gamma_d}$, $2d_de_d|e_\omega| = 2d_d\gamma_de_d\frac{|e_\omega|}{\gamma_d}$, 
	 and $2d_\beta|\beta||e_v| = 2d_\beta \gamma_\beta|\beta|\frac{|e_v|}{\gamma_\beta}$, $2d_\beta|\beta||e_\omega| = 2d_\beta \gamma_\beta|\beta|\frac{|e_\omega|}{\gamma_\beta}$ and
	completing the squares, one obtains
	\begin{align*}
		\bar{D}(|e_v|+|e_\omega|) \leq& 2d_d\gamma_de_d^2 + 2d_\beta\gamma_\beta\beta^2 + \left(\frac{d_d}{\gamma_d} + \frac{d_\beta}{\gamma_\beta} + 4d  \right)(e_v^2 + e_\omega^2) + 2D(|e_v|+|e_\omega|) \\
		=& 2d_d\gamma_de_d^2 + 2d_\beta\gamma_\beta\beta^2 + \ell_1(e_v^2 + e_\omega^2) + \ell_2(|e_v|+|e_\omega|)
	\end{align*}
	Hence, $\widetilde{W}$ becomes
	\begin{align*}
		\widetilde{W} \leq& -(k_d - 2d_d\gamma_d)e_d^2 - (k_\beta - 2d_\beta \gamma_\beta)\beta^2 - k_ve_v^2 - k_\beta e_\omega^2 - \hat{\ell}_1(e_v^2 + e_\omega^2) - \hat{\ell}_2(|e_v| + |e_\omega|) \\
		& + {\ell}_1(e_v^2 + e_\omega^2) + {\ell}_2(|e_v| + |e_\omega|) + \widetilde{\ell}_1 (e_v^2 + e_\omega^2)  + \widetilde{\ell}_2 (|e_v| + |e_\omega|)  \\
		=&-(k_d - 2d_d\gamma_d)e_d^2 - (k_\beta - 2d_\beta \gamma_\beta)\beta^2 - k_ve_v^2 - k_\beta e_\omega^2	=: Q(\widetilde{x})
	\end{align*}			
	Therefore, it holds that $\zeta \leq - Q(\widetilde{x})$, for all $\zeta\in \dot{\widetilde{V}}$, with $Q$ being a continuous and positive semi-definite function in $\mathbb{R}^{8}$, since $k_d - 2d_d\gamma_d>0$, $k_\beta - 2d_\beta \gamma_\beta > 0$. Choose now any finite $r>0$ and let $c < \min_{\|\widetilde{x}\|=r} W_1(\widetilde{x})$. Hence, all the conditions of Corollary \ref{th:LaSalle} are satisfied and hence, all Filippov solutions starting from $\widetilde{x}(0) \in \Omega_f \coloneqq \{ \widetilde{x} \in \mathcal{B}(0,r) : \widetilde{W}_2(\widetilde{x}) \leq c \}$ are bounded and remain in $\Omega_f$, satisfying $\lim_{t\to\infty} Q(\widetilde{x}(t)) = 0$. 
	
	Note that $r$, and hence $c$, can be arbitrarily large allowing and finite initial condition $\widetilde{x}(0)$. Moreover, the boundedness of $\widetilde{x}$ implies the boundedness of $e$, $\dot{e}$, $e_v$,  $\widetilde{\ell}_1$, and $\widetilde{\ell}_2$, and hence of $\hat{\ell}_1(t)$ and $\hat{\ell}_2(t)$, for all $t\in\mathbb{R}_{\geq 0}$. Finally, in view of the conditions $\frac{|\sin\phi\dot{p}_{\textup{d},1}-\cos\phi\dot{p}_{\textup{d},2}|}{e_d} < \alpha_1$ and $\frac{\sin\beta}{e_d} < \alpha_2$, one concludes the boundedness of all closed-loop signals, for all $t\geq 0$

\end{proof}

\section{Appendix} \label{app:B}

We provide here more  information on the experimental results of Section \ref{sec:exp res}. All the results were obtained on MATLAB environment using the ODE simulator. We performed the training of the neural networks in with the pytorch library in Python environment. The neural networks consist of $4$ fully connected layers of $512$ neurons; each layer is followed by a batch normalization module and a ReLU activation function. For the training we use the adam optimizer and the mean-square-error loss function. In all cases we choose a batch size of 256, and we train until a desirable average (per batch) loss of the order of $10^{-4}$ is achieved. All the values of the data used for the training were normalized in $[0,1]$.

Regarding the robotic-arm task, the points of interest are set as
$T_1 = [-0.15,-0.475,0.675, \frac{\pi}{2},0,0]^\top$, $T_2= [-0.6,0,2.5,0,-\frac{\pi}{2},-\frac{\pi}{2}]^\top$, $T_3= [-0.025,0.595,0.6,-\frac{\pi}{2},0,\pi]^\top$, and $T_4= [-0.525,-0.55,0.28,\pi,0,-\frac{\pi}{2}]^\top$,
and the corresponding joint-angle vectors as $c_1 = [-0.07, -1.05, 0.45, 2.3, 1.37, -1.33]^\top$, $c_2=
[1.28, 0.35, 1.75, 0.03, 0.1, -1.22]^\top$, $c_3= [-0.08, 0.85, -0.23, 2.58, 2.09, -2.36]^\top$, $c_4 = [-0.7, -0.76, -1.05, -0.05, -3.08, 2.37]^\top$ (radians).

We set a nominal value for the time intervals as $I_i = [0,20]$ (seconds), and we create 150 problem instances by varying the following attributes: firstly, we add uniformly random offsets in $[-0.3,0.3]$ (radians) to the elements of all $c_i$, and in $[-2,2]$ (seconds) to the right end-points of the intervals $I_i$; secondly, we add random offsets to the dynamic parameters of the robot (masses and moments of inertia of the robot's links and actuators) and we set a different friction and disturbance term $d(\cdot)$, leading to a different dynamic model in (\ref{eq:exp dynamics}); thirdly, we set different sequences of visits to the points $c_i$, $i\in\{1,\dots,4\}$, as dictated by $\phi$, i.e., one trajectory might correspond to the visit sequence $( (x(0),0) \to (c_1,t_{1_1}) \to (c_2,t_{1_2}) \to (c_3,t_{1_3}) \to (c_4,t_{1_4})$, and another to $( (x(0),0) \to (c_3,t_{1_3}) \to (c_1,t_{1_1}) \to (c_4,t_{1_4}) \to (c_2,t_{1_2})$. Finally, we add uniformly random offsets in $[-0.5,0.5]$ to the initial position of the robot (from the first point of the sequence), and we set its initial velocity randomly in the interval $[0,1]^6$.

Regarding the dynamics (\ref{eq:exp dynamics}), we use the methodology described in 
(\cite{siciliano2009modelling}) to derive the $B$, $C$, and $g$ terms. We set nominal link masses and moments of inertia as $m = [1,2.5,5.7,3.9,2.5,2.5,0.7]$ (kg) and $I = [0.02,0.04,0.06,0.05,0.04,0.04,0.01]$ (kgm$^2$), respectively, and we add random offsets in $(-\frac{m}{2},\frac{m}{2})$, $(-\frac{I}{2},\frac{I}{2})$ in the created instances.
Regarding the function $d()$ used in (\ref{eq:exp dynamics}); we set $d(\bar{x},t) = d_t(t) + d_f(\bar{x})$, where 
\begin{align*}
	d_t =  A_t \begin{bmatrix}
		\sin(\eta_1 t + \varphi_1) \\
		\vdots \\
		\sin(\eta_6 t + \varphi_6)
	\end{bmatrix}, \ \ \ 
	d_f = R_t A_t \dot{x}
\end{align*}
$A_t = \textup{diag}\{\ A_{t_i} \}_{i\in\{1,\dots,6\}}\in\mathbb{R}^{6\times 6}$, $A_{t_i}$ is a random term in $(0,2 m_i)$, $\eta_i$ is a random term in $(0,1)$, $\varphi_i$ is a random term in $(0,2)$, and $R_t \in \mathbb{R}^{6\times6}$ is a matrix whose rows are set as $R_{t_i}\dot{x}_i$ and where $R_{t_i}\in\mathbb{R}^{6\times6}$ is a diagonal matrix whose diagonal elements take random values in $\{0,1\}$. 
We chose the control gains of the control policy (\ref{eq:control as}) as $k_1 = 1, k_2 = 10$, and $k_{\ell_1}=k_{\ell_2}= 10$.  

Further experimental results are depicted in Figs. \ref{fig:error_dot}-\ref{fig:adaptation}; Fig. \ref{fig:error_dot}a depicts the mean and standard deviation of $\|e_v(t)\|$ of the proposed control policy and no-neural-network one for the 50 test instances, whereas Figs. \ref{fig:u_all}a depicts the control input that results from the control policy (\ref{eq:control as}) as well as the the neural-network output. Note that the control input converges to the neural-network output, i.e., $\lim_{t\to\infty}(u(t) - u_\textup{nn}(t))  = 0$, which can be also verified by (\ref{eq:control as}) and the fact that $\lim_{t\to\infty}e_v(t) = 0$. Fig. \ref{fig:adaptation}a depicts the mean and standard deviation of the adaptation signals $\hat{\ell}_1(t)$, $\hat{\ell}_2(t)$ for the 50 test instances. Finally, Fig. \ref{fig:ur5 2} shows timestamps of one of the test trajectories, illustrating the visit of the robot end-effector to the points of interest at the pre-specified time stamps.

\begin{figure}
		\centering
		\includegraphics[trim={0cm 0cm 0cm 0cm},width=.8\textwidth]{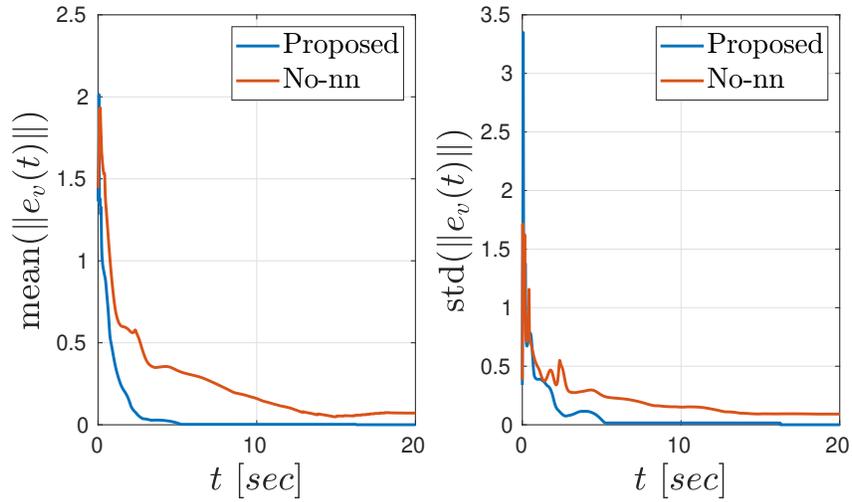}		
		\caption{Mean (left) and standard deviation (right) of ${e}_v(t)$ for the proposed and no-neural-network control policies.}
		\label{fig:error_dot}
\end{figure}

\begin{figure}
	\centering
	\includegraphics[trim={0cm 0cm 0cm 0cm},width=.95\textwidth]{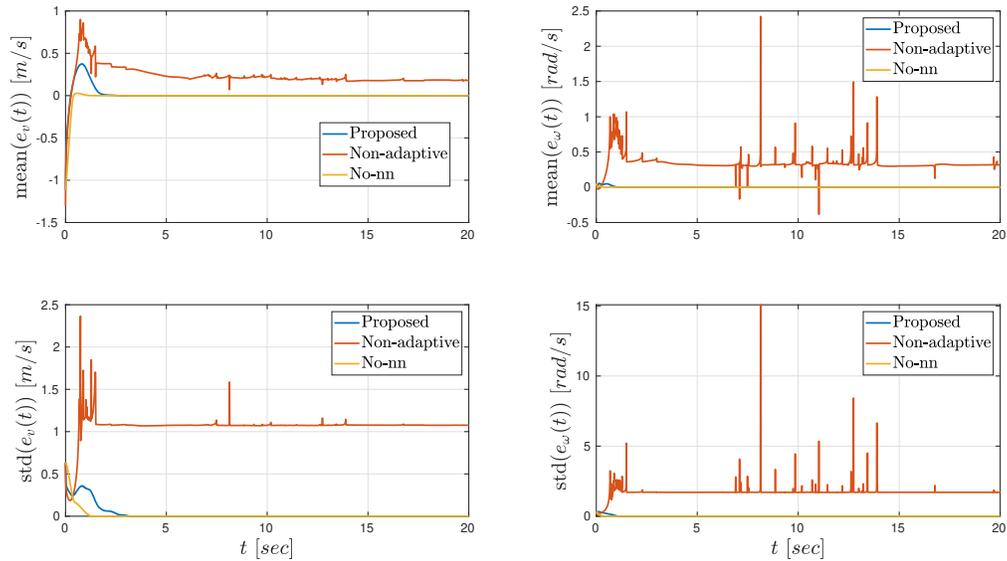}
	\caption{Mean (top) and standard deviation (bottom) of $e_v(t)$ and $e_\omega(t)$ for the proposed, non-adaptive, and no-neural-network control policies.}
	\label{fig:unicycle velocity errors}
\end{figure}


\begin{figure}
	\begin{subfigure}[b]{.5\textwidth}
		\centering
		\includegraphics[width=.85\textwidth]{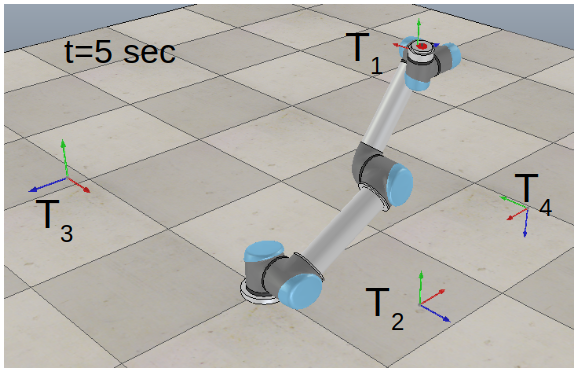}
		\caption{}
	\end{subfigure}
	~
	\begin{subfigure}[b]{.5\textwidth}
		\centering
		\includegraphics[height=.168\textheight,width=.85\textwidth]{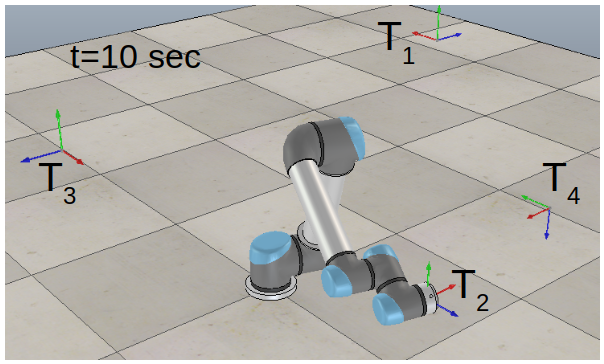}
		\caption{}
	\end{subfigure}
	~
	\begin{subfigure}[b]{.5\textwidth}
		\centering
		\includegraphics[height=.16\textheight,width=.85\textwidth]{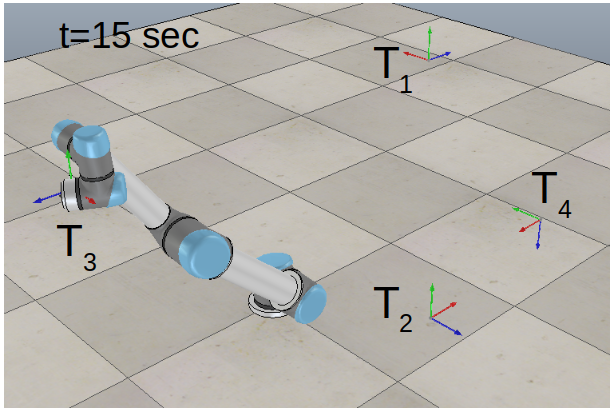}
		\caption{}
	\end{subfigure}
	~
	\begin{subfigure}[b]{.5\textwidth}
		\centering
		\includegraphics[width=.85\textwidth]{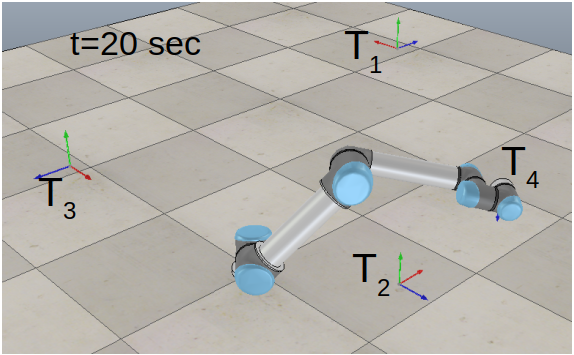}
		\caption{}
	\end{subfigure}
	\caption{Illustration of the execution of one of the test trajectories of the UR5 robotic arm, visiting the points of interest at the pre-specified time stamps.}
	\label{fig:ur5 2}
\end{figure}

Regarding the unicycle experiments, the dynamic terms in (\ref{eq:dynamics unicycle}) have the form 
\begin{align*}
&	M = \begin{bmatrix}
		M_1 & M_2 \\
		M_2 & M_1
	\end{bmatrix} \\	
&	f_\theta(\bar{x},t) = d(\bar{x},t)
\end{align*}
with $M_1 \coloneqq \frac{mr^2}{4} + \frac{(I_C + md^2)r^2}{4R^2} + I_0$, $M_2 \coloneqq \frac{mr^2}{4} - \frac{(I_C + md^2)r^2}{4R^2}$, and where $I_C$ is the moment of inertia of the vehicle with respect to point C (see Fig. \ref{fig:unicycle}), $I_0$ is the moment of inertia of the the wheels, and $d$ is the distance of the between point C and the vehicle's center of mass $(p_1,p_2)$. The term $d(\bar{x},t)$ is chosen as in the robotic-manipulator case. The points to visit are chosen here as $c_1 = [0,0]^\top$, $c_2 = [0,2]^\top$, $c_3 = [2,0]^\top$, $c_4 = [2,2]^\top$ and $I_i$ is chosen as $I_i=[0,20]$, $i\in\{1,\dots,4\}$.
We derive 150 problem instances by varying the following attributes: we vary $c_i$ with random offsets in $[-0.3,0.3]$ and the right end-points of $I_i$ with random offsets in $[-2,2]$; we add random offsets to the dynamic parameters (elaborated subsequently) and the function $d(\bar{x},t)$, and we set different sequences of visits to the points $c_i$; we further vary the unicycle's initial position from the first $c_i$ with random offsets in $[-0.3,0.3]$, and its initial orientation with random offsets in $[-0.25,0.25]$ (rad) from $\theta(0) = \arctan(e_2(0)/e_1(0))$; we further set random values in $[-0.25,0.25]$ (rad/s) to the initial wheel velocities. 

We set nominal values for the dynamic and geometric parameters, appearing in the inertia matrix, as $m = 28$ (kg), $I_C = 0.1$ (kgm$^2$), $I_0 = 0.01$ (kgm$^2$), $r = 0.01$ (m), $d=0.01$ (m), and we added random offsets in $(-\frac{m}{2},\frac{m}{2})$, $(-\frac{I_A}{2},\frac{I_A}{2})$, $(-\frac{I_0}{2},\frac{I_0}{2})$, $(-\frac{r}{2},\frac{r}{2})$, $(-\frac{d}{2},\frac{d}{2})$, respectively, for the problem instances. We chose the control gains of (\ref{eq:control as unicycle}) as $k_d=0.25$, $k_\beta=1$, $k_v=k_\omega = k_{\ell_1} = k_{\ell_2} = 10$, $k_v = k_\omega = 1$. 
The non-adaptive control policy we compared our algorithm with was selected as   
\begin{align*}
	 u_S &= M_S \dot{v}_\textup{d} - k_ve_v  + e_d \cos\beta - \beta\frac{\sin\beta}{e_d} \\
	 u_D &= M_D \dot{\omega}_\textup{d} - k_\omega e_\omega + \beta
\end{align*}
where $v_\textup{d}$, $\omega_{\textup{d}}$ are chosen as (\ref{eq:v_des unicycle}) and
$M_S$, $M_D$ are static estimates of $2\frac{M_1+M_2}{r}$, $2R\frac{M_1-M_2}{r}$, respectively. More specifically, we set a deviation of $25\%$ in the parameters appearing in these terms to form $M_S$ and $M_D$. Further experimental results are depicted in Figs. \ref{fig:error_dot}-\ref{fig:adaptation}; Fig. \ref{fig:error_dot}b depicts the mean and standard deviation of the velocity errors $e_v(t)$, $e_\beta(t)$ for the 50 test instances, showing convergence to zero, whereas Figs. \ref{fig:u_all}b depicts the control input that results from the control policy (\ref{eq:control as}) as well as the the neural-network output. Similarly to the robotic-manipulator case, the control input converges to the neural-network output. Finally, Fig. \ref{fig:adaptation}b depicts the mean and standard deviation of the adaptation signals $\hat{\ell}_v(t)$, $\hat{\ell}_\beta(t)$, $\hat{\ell}_1(t)$, $\hat{\ell}_2(t)$ for the 50 test instances.

\begin{figure}
	\begin{subfigure}[b]{0.5\textwidth}
		\centering
		\includegraphics[trim={0cm 0cm 0cm 0cm},width=\textwidth]{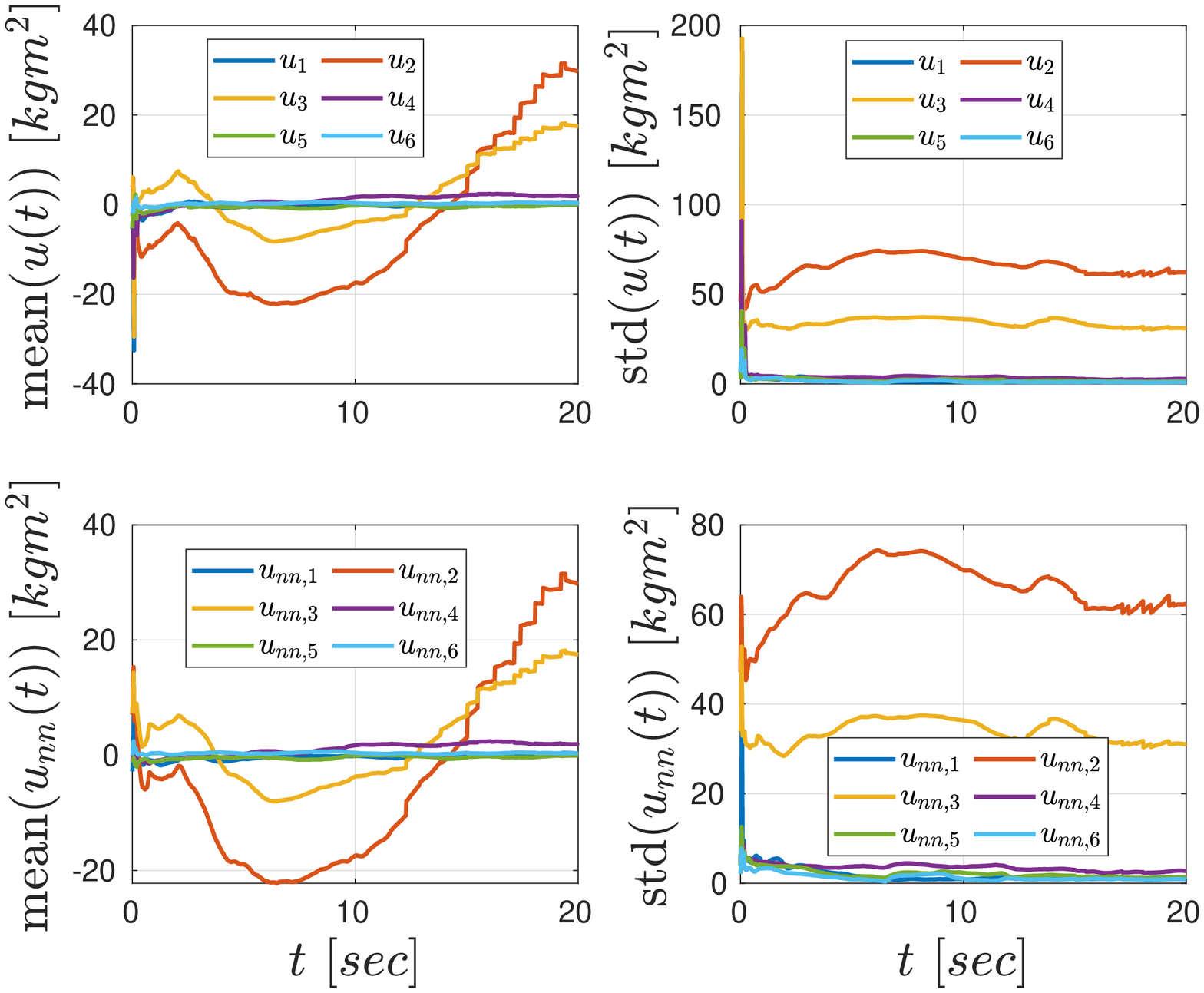}		
		\caption{}
	\end{subfigure}
	~
	\begin{subfigure}[b]{0.5\textwidth}
		\centering
		\includegraphics[trim={0cm 0cm 0cm 0cm},width=\textwidth]{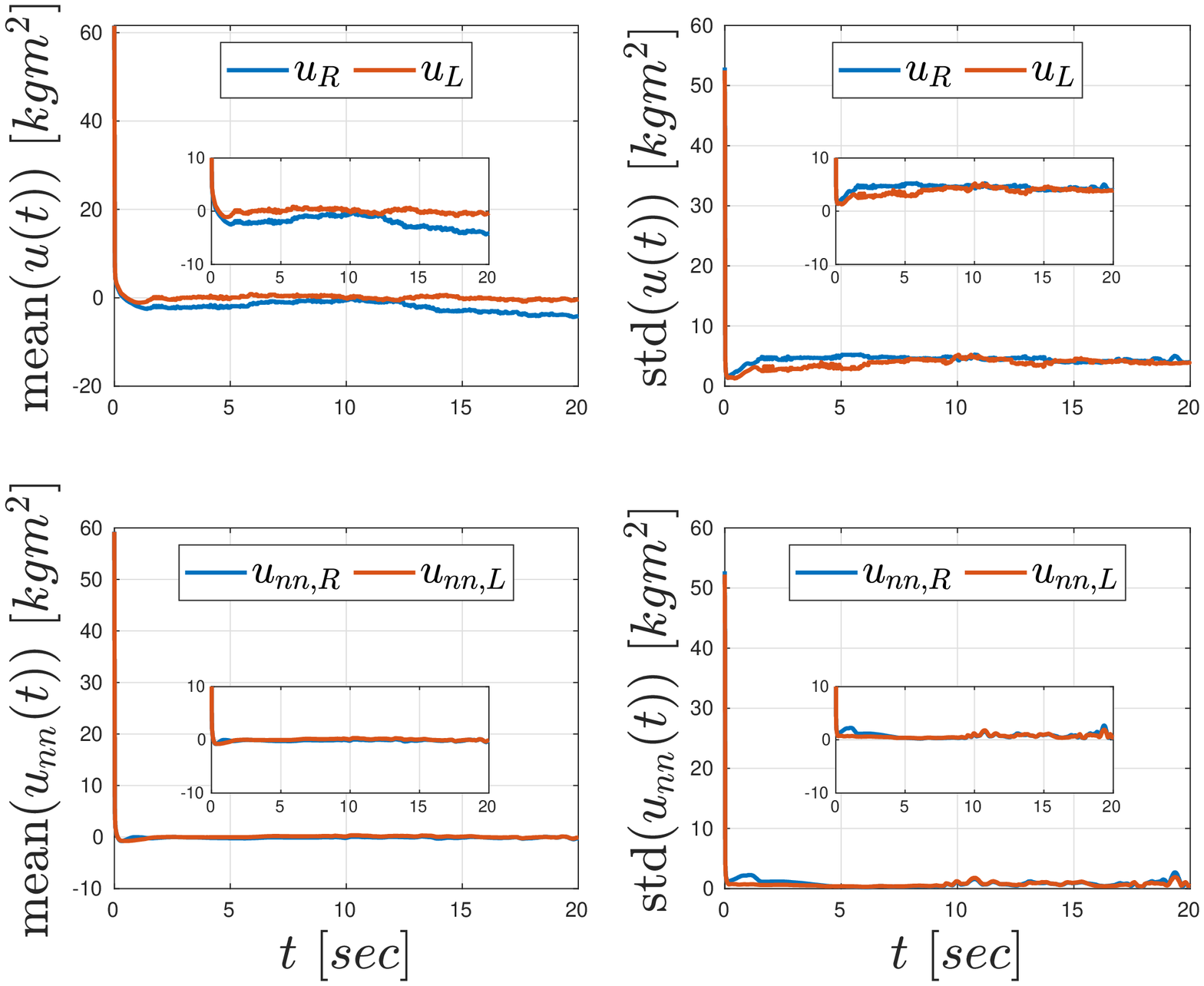}
		\caption{}
	\end{subfigure}
	\caption{(a): Mean (left) and standard deviation (right) of $u(t)$ (top) and $u_\textup{nn}(t)$ (bottom) for the proposed control policy and the robotic-manipulator environment. (b): Mean and standard deviation (right) of $u(t)$ (top) and $u_\textup{nn}(t)$ (bottom) for the proposed control policy and the unicycle environment.}
	\label{fig:u_all}
\end{figure}

\begin{figure}
	\begin{subfigure}[b]{0.5\textwidth}
		\centering
		\includegraphics[trim={0cm 0cm 0cm 0cm},width=.9\textwidth]{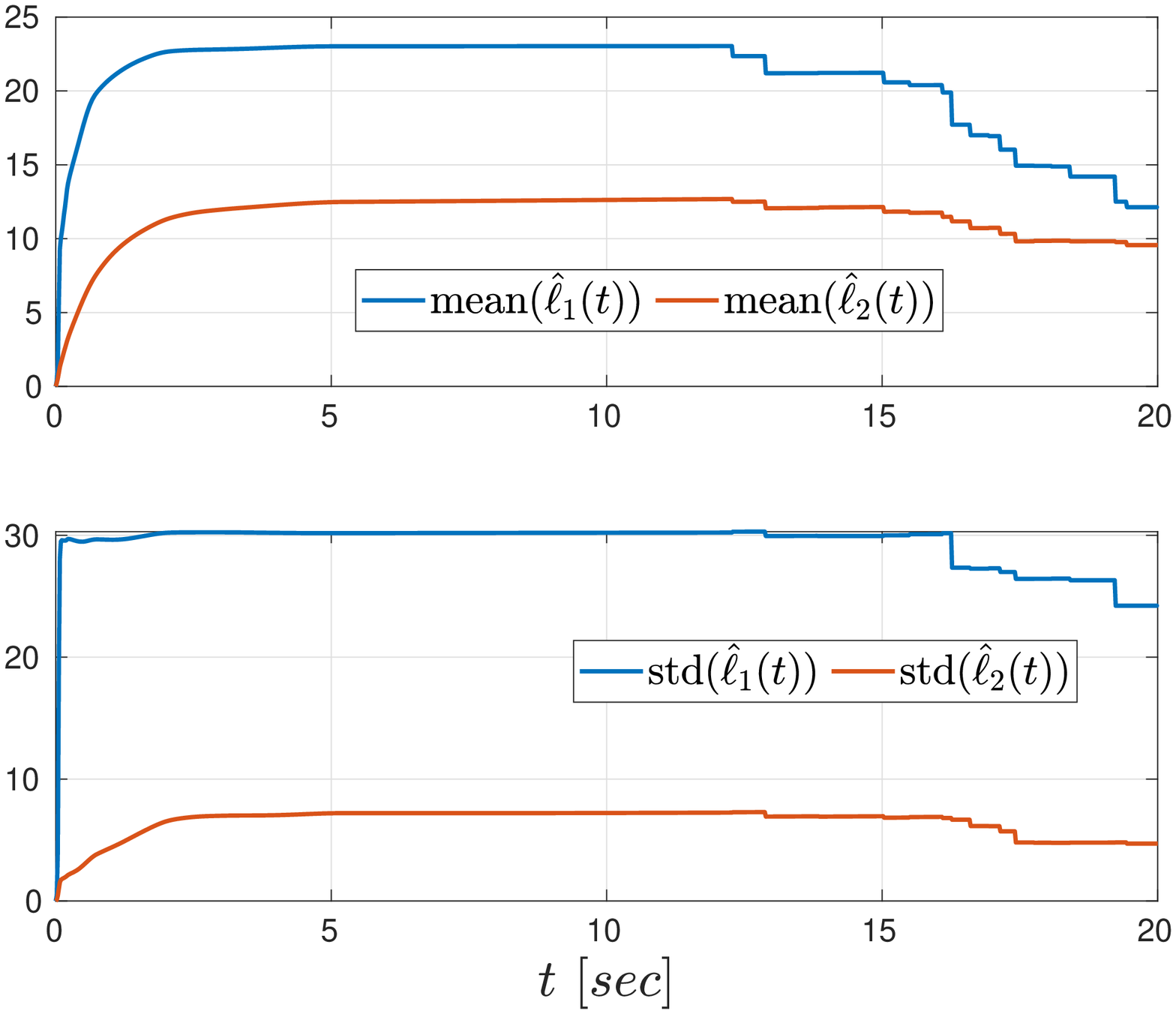}		
		\caption{}
	\end{subfigure}
	~
	\begin{subfigure}[b]{0.5\textwidth}
		\centering
		\includegraphics[trim={0cm 0cm 0cm 0cm},width=.9\textwidth]{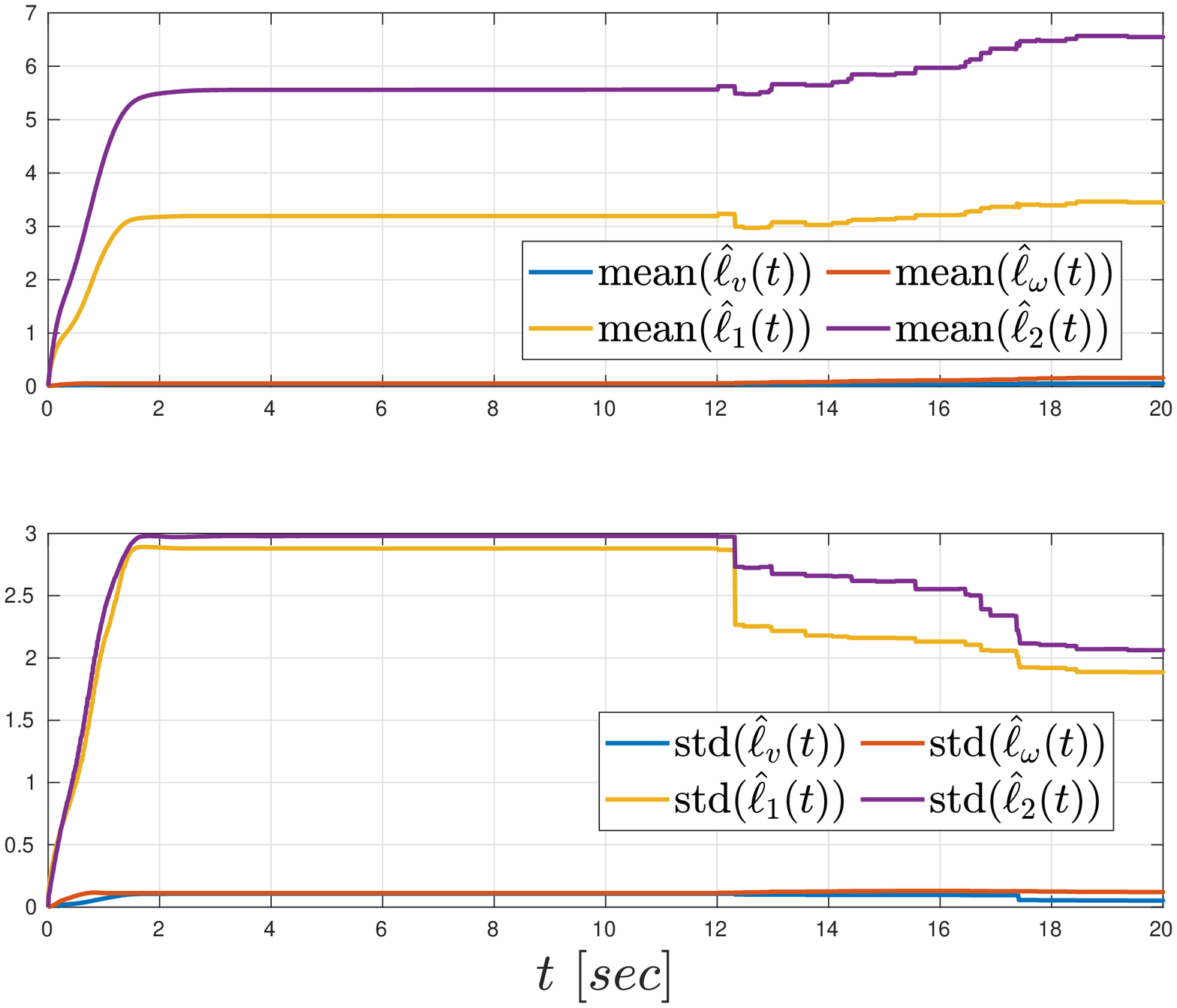}
		\caption{}
	\end{subfigure}
	\caption{(a): Mean (top) and standard deviation (bottom) of $\hat{\ell}_1(t)$, $\hat{\ell}_2(t)$ for the proposed control policy and the robotic-manipulator environment. (b): Mean (top) and standard deviation (bottom) of $\hat{\ell}_v(t)$, $\hat{\ell}_\omega(t)$, $\hat{\ell}_1(t)$, $\hat{\ell}_2(t)$ for the proposed control policy and the unicycle environment.}
	\label{fig:adaptation}
\end{figure}

Finally, regarding the pendulum environment, we consider the dynamics
\begin{align*}
	\ddot{q} = \frac{g}{L}\sin q + u + d(t)
\end{align*}
where $L$ is the pendulum's link length, $g$ is the gravitational constant, and $d(t)$ is a term of exogenous disturbances. We created 150 instances by varying the mass and link length of the pendulum, the external disturbances, and setting its initial position and velocity randomly in $[-1,1]$ (rad) and $[-1,1]$ (rad/s), respectively. We selected unitary value for its nominal link length, and we added random offsets in $(-0.5,0.5)$ in each instance. We set the external disturbances  as $d = A \sin(\eta t + \phi)$, with $A$, $\eta$, and $\phi$ taking random values in $[0,0.2]$, $[0,1]$, $[0,2]$, respectively. In contrast to the robotic manipulator case, we assume feedback of $\sin q$, $\cos q$, $\dot{q}$, and set the error $e$ in (\ref{eq:v d}) as $1-\cos(q-\pi)$. Finally, we chose the control gains as $k_1=k_2=1$, and $k_{\ell_1} = k_{\ell_2} = 10$.

\end{document}